\documentclass[10pt,journal,compsoc]{IEEEtran}
\usepackage[nocompress]{cite}
\usepackage{amsmath,amssymb,amsfonts}
\usepackage{amsthm}
\usepackage{mathtools}
\usepackage{algorithmic}
\usepackage{graphicx}
\usepackage{tikz}
\usetikzlibrary{arrows.meta,positioning,fit}
\usepackage{textcomp}
\usepackage{xcolor}
\usepackage{pifont}   
\usepackage{enumitem} 
\usepackage{booktabs}
\usepackage{tabularx}
\usepackage{makecell}
\usepackage{listings}
\usepackage[hyphens]{url}
\usepackage{seqsplit}

\usepackage{microtype}
\usepackage[colorlinks=true,linkcolor=black,citecolor=black,urlcolor=black]{hyperref}

\newtheorem{property}{Property}
\newtheorem{definition}{Definition}
\newtheorem{lemma}{Lemma}
\newtheorem{remark}{Remark}
\newtheorem{proposition}{Proposition}
\newcommand{\sys}{\textsc{Remit}}
\newcommand{\contract}{\textsc{Resume Contract}}
\newcommand{\LG}{LangGraph}
\newcommand{\LI}{LlamaIndex Workflows}
\newcommand{\CA}{CrewAI}
\newcommand{\yes}{\ding{51}}
\newcommand{\no}{\ding{55}}
\newcommand{\docd}{\textsf{D}}
\newcommand{\und}{\textsf{U}}
\DeclareFontShape{OT1}{ppl}{m}{scit}{<->ssub*ppl/m/sc}{}
\begin{document}
\title{Resume Means Resume: A Machine-Checked\\ Conformance Contract for Checkpoint,\\ Interrupt, and Resume Semantics in\\ Workflow Persistence Layers}
\author{Sajjad~Khan
\thanks{S. Khan is an independent researcher, London, UK
(e-mail: sajjadanwar200@gmail.com).}%
\thanks{Artifact (conformance probes, TLA\textsuperscript{+} models and TLC
configurations, and the \sys{} reference-sequencer design), with a
single-command audit (\texttt{reproduce.sh}) that re-derives every headline
number from committed data. The artifact repository and the \sys{} package
sources (Rust core, PyO3 bindings, decision-free \LG{} shim) are private
pending publication and will be released publicly on publication; access
for review is available from the author on request.}%
}

\IEEEtitleabstractindextext{%
\begin{abstract}
A framework that persists execution state so a run can be interrupted,
survive a crash, and continue must decide what a resume means for
effects that already happened. Five widely deployed agent workflow
frameworks answer differently, none exposes a machine-checkable
contract, and measured behavior violates even the fragments they state.
The \contract{} states six properties over the persistence API ---
prefix continuation, effect exactly-once, fork determinism, checkpoint
validity, consume-once, recovery determinism --- plus fork-intent and
liveness obligations. A TLA\textsuperscript{+} model checks a reference
semantics exhaustively, verdicts unchanged at scaled bounds
($7.4{\times}10^6$ distinct states), and the reference conjunction is
additionally TLAPS-proved unbounded (196 obligations); a 39-cell fault
matrix and two
companion modules yield the separating models independence requires
(five properties independent of the conjunction of the others;
consume-once splits, its consumption clause independent of all six).
A deterministic, LLM-free harness measures five frameworks at pinned
releases. \LG{}~1.2.9
durably records a second resume value and never consults it, persists
schema-invalid state silently, and re-executes durably recorded work
after a real \texttt{SIGKILL} --- exactly-once across interrupts,
at-least-once across crashes, on one API. \CA{}~1.15.2 re-executes
completed effect-bearing methods against its written claim, and
pydantic-graph~1.x cannot resume after a mid-node crash; no two probed
frameworks share a conformance profile. Consume-once holds sequentially
and fails under concurrent delivery: $k$ processes resuming one parked
interrupt fire the gated effect $k$ times, saturation $1.0$ in 36 of 40
cells and never below $0.933$ on both durable backends. The admitting
window tracks the gated node's own execution time, measured by
dose--response. The failure crosses hosts, two racers on separate
machines duplicating in $10/10$ repetitions. Live-model cells
reproduce the fork violation $40/40$ per model. \sys{}, a reference
sequencer whose Verus-verified recovery core is line-identical to the
shipped executable under a CI gate (no end-to-end refinement is
claimed), repairs the fork and validity cells. The cross-process cell is
repaired at the read path, and the repair ships: an opt-in gate claims
consumption in the shared store, serving one racer and refusing the
rest before any node executes ($\{1{:}10\}$ on both durable backends,
and with the racers on two hosts).
\end{abstract}

}

\maketitle
\IEEEdisplaynontitleabstractindextext
\IEEEpeerreviewmaketitle

\IEEEraisesectionheading{\section{Introduction}\label{sec:intro}}

\IEEEPARstart{A}{ persistence} plane makes a simple promise: a run can
stop --- for a human approval, for a crash, for a preemption --- and
then continue. The promise is old and its failure modes are classical;
what is new is where the promise is now being made. A generation of
workflow frameworks for LLM agents has shipped that machinery to
developers who gate payments, messages, and file writes behind it:
\LG{} checkpoints graph state per superstep and exposes
\texttt{interrupt}/\texttt{Command(resume=...)}; \LI{} serializes a
workflow \texttt{Context} mid-run and restores it; \CA{} persists flow
state and restores runs from checkpoint files. As agents are given
authority over non-idempotent effects, ``continue'' becomes a
correctness question: which completed effects may fire again, what
happens when the same interrupt is answered twice with different
values, what may be persisted, and whether the recovery decision is a
function of durable state. This paper measures whether the machinery
holds.

The gap is not mere absence of specification but incoherence.
None of the evaluated frameworks exposes a machine-checkable contract for
these obligations, and where fragments \emph{are} stated, the frameworks
contradict one another: \CA{}'s checkpointing documentation claims
restoration ``resume[s] without re-running completed
work''~\cite{crewai-ckpt-docs}; \LI{} instructs users to place
\texttt{wait\_for\_event} early and ``make any preceding work safe to
re-execute''~\cite{li-context-docs}; \LG{} memoizes completed
\texttt{@task} results across resume. Three frameworks, three
incompatible answers --- and, as we show, two of the three do not satisfy
even the semantics they themselves state or imply.

The chain from divergence to harm is concrete: a developer who ports a
side-effecting workflow across frameworks cannot locally determine which
discipline is in force --- no type, no signature, no documented property
to consult --- and the GitHub issues this paper reproduces are that gap
made real. We treat the issues as motivation, never measurement; every
behavioral claim is re-established by the deterministic harness of
Section~\ref{sec:method}. The persistence plane lacks what network
protocols and file systems have had for decades: an explicit contract, a
machine-checked model, and a conformance suite. A single contract can
span divergent execution models because they share a boundary ---
checkpoints, interrupts, resume commands, a state-inspection API --- and
every property below is stated over that surface alone; a framework
documenting a weaker discipline is recorded as divergent (\docd{}), not
defective. We make that concrete in four steps.

\textbf{(1) The \contract{}} (Section~\ref{sec:contract}). Six properties
over an abstract resume plane: \emph{prefix continuation} (PC),
\emph{effect exactly-once} (EO), \emph{fork determinism} (FD),
\emph{checkpoint validity} (CV), \emph{consume-once} (CO), and
\emph{recovery determinism} (RD). PC/EO govern what may re-execute; FD/CO
govern the interrupt lifecycle; CV governs what may be persisted; RD ---
motivated by an ordering defect we reported in \LG{}'s synchronous
durability mode (issue \#8039) --- requires the recovery decision to be a
function of durable state. The property set predates the measurements
(Section~\ref{sec:independence}).

\textbf{(2) A machine-checked model} (Section~\ref{sec:formal}). A
TLA\textsuperscript{+} module with six fault switches, each modeling a
violation class observed in a deployed framework. TLC verifies the
reference configuration against all six properties (87 states generated,
59 distinct, no error) and produces a depth 4--6 counterexample per
single-fault configuration. A per-invariant matrix --- every fault
against every property, 39 runs --- maps each fault's full violation
footprint; the footprints are discovered, not stipulated, and the runs
whose entire state space stays clean are exactly the separating models
Section~\ref{sec:independence} needs.

\textbf{(3) A deterministic conformance harness and cross-framework study}
(Sections~\ref{sec:method}--\ref{sec:pilot}). Every probe is a
pure-Python protocol sequence over framework persistence APIs --- no LLM
calls, no timing windows; crashes are exception-based in the matrix and
replicated under a barrier-synchronized \texttt{SIGKILL} (probe~133) ---
with process-local effect counters, cross-checked on the durable-backend
probes by an on-disk external ledger. On the current releases (\LG{}~1.2.9, \LI{}~2.22.2, \CA{}~1.15.2,
pydantic-graph~1.107.1, AutoGen AgentChat~0.7.5) the study finds: a
live fork violation and a now-silent validity violation in \LG{} on
all three checkpointer backends including live PostgreSQL; re-execution
of completed effect-bearing methods under checkpoint restore in \CA{},
against the feature's written claim; documented at-least-once prefix
replay in \LI{}; a crash that defeats pydantic-graph's own resume
entry point; loud rejection of tampered state in AutoGen AgentChat,
alone among probed frameworks; every incomplete persistence boundary
licensing re-execution of completed work (probe~160); consume-once
holding sequentially and failing under cross-process delivery
(probe~159); and no two probed frameworks sharing a conformance
profile. Two probed \LG{} behaviors (\#7361, \#6792) shipped as 1.1.x
regressions and are fixed in 1.2.9 --- absent a stated contract,
semantics drift even within one framework.

\textbf{(4) A repair path with a verified model, shipped}
(Section~\ref{sec:remit}). \sys{} is a reference resume sequencer and
append-only effect ledger beneath existing checkpointer interfaces
(\LG{}'s \texttt{BaseCheckpointSaver} first), delivered at labeled
maturity levels: invariants stated in Verus~\cite{verus} and
machine-discharged, the shipped recovery-decision core a verified
executable function (no end-to-end refinement claimed); a Rust core
behind PyO3 bindings, conformance-checked against the
TLA\textsuperscript{+} transition relation; a validity gate that turns
silent invalid persistence into loud rejection live; and fork
determinism settled by a matched pair --- write-path keying fails, a
read-path fork-intent filter repairs \#6663 on the identical protocol
--- re-established by the packaged decision-free shim with the stock
saver as differential control. This completes the mechanism account of
Section~\ref{sec:derived}.

\paragraph*{Positioning} Two recent systems bracket this paper
without occupying it: below the framework API, Crab~\cite{crab}
checkpoints OS-level sandbox state; above it, DART~\cite{dart} decides
when a mechanically possible rollback is semantically admissible.
Neither asks whether the primitives themselves keep their promises. Consent-integrity
work~\cite{cim} binds an approval to the \emph{content} of the action;
FD/CO govern the approval \emph{lifecycle}. Concurrent framework-testing
work~\cite{logichunter} searches for bugs; we define the contract such
bugs violate, measure conformance, and provide a reference implementation
with a machine-verified model. To our knowledge no prior work combines an
explicit resume contract, a machine-checked model, and cross-framework
conformance measurement.

\emph{What is and is not claimed.} Because the sections below qualify
each claim where it is made, the scope is collected once here.
\emph{Proved}: the FD--CO impossibility without a discriminator
(Lemma~\ref{prop:fdco}); independence of EO, PC, FD, CV, and RD
from the conjunction of the others, and CO's structural dependence on
EO (Proposition~\ref{prop:independence}), each witnessed by a model
checked exhaustively at stated bounds; and the reference conjunction
itself, TLAPS-discharged unbounded over every constant assignment
satisfying the module's assumptions (Sec.~\ref{sec:threats}).
\emph{Machine-checked}: the six
properties over the reference model's full reachable space, re-checked
at the scaled bounds of Table~\ref{tab:tlc}; the per-invariant fault
matrix at both bound sets; the independence witnesses.
\emph{Measured}: conformance on the probed paths of five frameworks at
pinned versions, with the effect oracle a durable cross-process ledger.
\emph{Verified}: \sys{}'s recovery-decision core, an executable
function line-identical to the shipped one under a CI gate; the EO
admission and PC/CV commit cores as verified executable twins bridged
to the shipped implementations by an exhaustive differential suite
(Sec.~\ref{sec:package}) --- not the composite package, and no
refinement to the compiled binary. \emph{Not claimed}: minimality or
completeness of the property set; prevalence rates for any violation;
that the probed frameworks represent the ecosystem; that any framework
fails a property it is not measured on.

\section{The Resume Plane and Its Fragmentation}\label{sec:plane}

\subsection{Primitives}
We use \emph{resume plane} for the subsystem of an agent framework that
(i)~durably records execution progress (\emph{checkpoints}), (ii)~parks a
run pending external input (\emph{interrupts}), and (iii)~continues a run
from durable state (\emph{resume}), whether after an interrupt, a crash, or
an explicit restore. In \LG{} the plane is the checkpointer
(\texttt{InMemorySaver}, \texttt{SqliteSaver}, \texttt{PostgresSaver}) plus
\texttt{interrupt()} and \texttt{Command(resume=...)}, addressed by \texttt{(thread\_id, checkpoint\_id)}; the persistence
documentation gives invocation with a prior \texttt{checkpoint\_id}
two readings --- branch-creating time travel (``fork the graph state at
arbitrary checkpoints'') and replay with interrupts
re-triggered~\cite{lg-persist-docs}. In \LI{} the plane is
\texttt{Context.to\_dict()}/\texttt{from\_dict()} plus the
\texttt{wait\_for\_event} idiom with \texttt{InputRequiredEvent}/%
\texttt{HumanResponseEvent}. In \CA{} it is \texttt{@persist} flow-state
persistence and, since the checkpointing feature, \texttt{CheckpointConfig}
with \texttt{Flow.from\_checkpoint}.

\subsection{Three frameworks, three written semantics}
Table~\ref{tab:frag} summarizes what each framework's own documentation
says about the central question --- does completed work re-execute on
resume? --- for the mechanism a practitioner would reach for, and the
statements are mutually incompatible: \CA{}'s checkpointing claims
completed work is skipped; \LI{} documents that step code before a
durable wait replays and that in-progress steps restart on
restore~\cite{li-hitl-docs}; \LG{} memoizes completed \texttt{@task}
results. A developer porting a side-effecting workflow between
frameworks silently crosses from an exactly-once regime into an
at-least-once regime with no type error, no warning, and no named
property to consult.

\begin{table}[t]
\caption{Documented resume discipline for completed work, per framework
(mechanism a practitioner would use for durable human-in-the-loop or crash
resume). ``Stated'' quotes or paraphrases the framework's own
documentation; Section~\ref{sec:pilot} tests whether behavior matches.}
\label{tab:frag}
\centering
\footnotesize
\begin{tabularx}{\columnwidth}{@{}>{\raggedright\arraybackslash}p{0.32\columnwidth}X@{}}
\toprule
Framework (mechanism) & Stated discipline for completed work on \mbox{resume} \\
\midrule
\CA{}~1.15.2 (\texttt{Checkpoint\allowbreak Config}) &
Exactly-once: restore ``resume[s] without re-running completed
work''~\cite{crewai-ckpt-docs}. \\
\addlinespace[2pt]
\LI{}~2.22.2 (\texttt{wait\_for\_event} + \texttt{Context} serialization) &
At-least-once for step prefixes: the runtime replays the step when the
event arrives; users must ``make any preceding work safe to
re-execute''~\cite{li-context-docs}; in-progress steps restart on
restore~\cite{li-hitl-docs}. \\
\addlinespace[2pt]
\LG{}~1.2.9 (\texttt{@task} + checkpointer) &
Exactly-once for task bodies via memoization of completed task results
across resume; graph-node semantics are checkpoint-granular~\cite{lg-persist-docs}. \\
\bottomrule
\end{tabularx}
\end{table}

\subsection{Version stability of the violations}\label{sec:sweep}
The \LG{} probe executed across five releases --- 1.0.5, 1.1.0, 1.1.3,
1.1.10, 1.2.9 --- to separate stable violations from transient
regressions. The fork violation (\#6663) and the silent-persistence
violation (\#6491 class) reproduce at \emph{every} tested version: stable
properties of the resume plane across a year of releases, not
regressions awaiting a patch. For \#7361 and \#6792, which the trackers
report as regressions in narrow windows, the sweep finds
current-and-adjacent versions clean; those are relied on only as
documented history, never as our own measurements. The evidence file
\texttt{results/regression/} records every cell.

\subsection{Layer positioning}
Figure~\ref{fig:layers} places the contract between the two strong recent
systems in the recovery space. Crab~\cite{crab} operates below the
framework API: an eBPF-instrumented host runtime that checkpoints
OS-visible sandbox state, and whose deployment machinery synthesizes
cached LLM responses so a restored agent does not replay completed
actions --- an implicit acknowledgment, at another layer, of the property
this paper names EO. DART~\cite{dart} operates above the API, certifying
when a rollback is admissible under committed downstream consumers. Both
take the primitive's own semantics as given; the contract layer asks the
prior question: \emph{is the primitive sound?} Mature recovery
disciplines surround this layer and are positioned in
Section~\ref{sec:related}; none specifies or measures the agent-framework
resume API itself.

\begin{figure}[t]
\centering
\begin{tikzpicture}[
  box/.style={draw, rounded corners=2pt, align=center, inner sep=5pt,
              font=\footnotesize, minimum width=0.9\columnwidth},
  lbl/.style={font=\scriptsize\itshape, anchor=west}
]
\node[box, fill=gray!8] (dart)
  {DART~\cite{dart}: rollback \emph{admissibility} over primitives\\
   (committed downstream consumers, effect policies)};
\node[box, below=5pt of dart, fill=blue!6, thick] (rc)
  {\textbf{This paper: the \contract{}}\\
   PC\,/\,EO\,/\,FD\,/\,CV\,/\,CO\,/\,RD over the framework resume API\\
   (checkpointers, interrupts, resume commands)};
\node[box, below=5pt of rc, fill=gray!8] (crab)
  {Crab~\cite{crab}: OS-level sandbox checkpoint/restore\\
   (filesystems, processes, microVMs)};
\end{tikzpicture}
\caption{Layer positioning. Adjacent systems assume the framework resume
primitive is sound; this paper specifies and tests that assumption.}
\label{fig:layers}
\end{figure}

\emph{Comparability and reading guide.} Per framework we probe the
officially documented mechanism for durable human-in-the-loop or crash
recovery (the artifact's selection manifest records adoption figures and
retrieval dates), instantiating the \emph{same} abstract workflow on
each plane --- gate a non-idempotent effect on a human decision, crash
after a durable step --- so divergence between matrix rows is precisely
the hazard a porting developer inherits. Documented semantics (this
section) and measured behavior (Section~\ref{sec:pilot}) are kept
distinct by the classification rule of Table~\ref{tab:matrix} note~b.
Longitudinal depth follows behavioral richness: \LG{} is swept across
five releases; the remaining headline cells are release-swept in the artifact
(\texttt{results/sweep/}, release-date receipts): six \CA{} releases
spanning 2026-04-08 to 2026-07-24, the CheckpointConfig divergence
present at every release where the restore path exists --- coeval with
the feature; three \LI{} and three pydantic-graph~1.x releases, the
parallel 2.x line no longer exposing the probed module. Stable
violations under rapid release churn: nothing converges toward an
implicit norm.

\section{The Resume Contract}\label{sec:contract}

\emph{Design principles.} The property set is not a bug taxonomy
ordered after the fact; they answer, in order, the questions any caller
who needs deterministic resume semantics --- any caller gating
non-idempotent effects on the plane --- must be able to answer when
crossing an interrupt/crash/resume boundary:
\emph{where} does execution resume (PC)? do effects \emph{repeat} (EO)?
what does a \emph{fork} mean (FD)? what may be \emph{persisted} (CV)?
what \emph{consumes} authority (CO)? is recovery a function of
\emph{durable state alone} (RD)? --- plus, one layer down, \emph{how is a
fork asked for} (FI, a protocol obligation). The observed failures of
Section~\ref{sec:pilot} instantiate these questions; they did not
generate them.

\subsection{Abstract model}
\begin{definition}[Resume plane]\label{def:plane}
A \emph{run} executes tasks $1..N$ in order; task $t$ carries one
non-idempotent external effect $e_t$. One task $\mathit{IP}$ is
\emph{interrupt-gated}: its effect fires only after a resume value
$v\in V$ is consumed. Completing task $t$ appends a \emph{checkpoint
record} $\langle t,\mathit{valid}\rangle$ to a durable log and advances
the \emph{durable frontier} $F$ to $\max(F,t)$. A \emph{crash} erases
volatile state; a \emph{recovery} chooses a continuation point as a
function of the durable log. A \emph{resume} addressed to the interrupt
checkpoint carries a value $v_k$ and yields a branch outcome
$o_k$.
\end{definition}

The model deliberately abstracts framework detail --- graph supersteps,
event queues, flow listeners --- to the observable interface a caller
programs against: effects, checkpoints, interrupts, resume values,
outcomes, recovery decisions. One gated task is the minimal nontrivial
instance; composition across gates and parallel branches is treated under
\emph{Scope} at the end of this section and probed in
Section~\ref{sec:pilot}. Each property below is stated over that
interface, so verdicts are decided from surface observables alone;
probes that instrument beneath it (adversarial savers, executor traces,
the source audit) localize mechanisms, never decide cells, and PC's
log-derivation clause is certified, on any finite probe set, through
its observable consequences --- state equality with the effect record
intact --- not by observing provenance.

\subsection{Properties}
\begin{property}[PC: Prefix continuation]\label{p:pc}
Recovery continues from the durably recorded frontier state: execution
after recovery begins in the state $S_F$ recorded at frontier $F$, or in a
state re-derived deterministically from the durable log alone that equals
$S_F$. \emph{Memoized replay conforms}: prefix code may re-run during
recovery provided every prefix effect is served from the durable record
(so EO is preserved) and the re-derived state is a pure function of the
log. Re-deriving state from initial values, or traversing the prefix
against live effects, violates PC. Equality here is \emph{observable-state}
equality, and \emph{observable state} is, throughout this paper, exactly
what is retrievable through the framework's public state-inspection API by
the caller and by subsequent tasks; internal representation is
unconstrained.
\end{property}
\begin{property}[EO: Effect exactly-once]\label{p:eo}
For every task $t$, effect $e_t$ fires at most once on a branch across any
sequence of interrupts, crashes, and resumes (as a safety invariant EO is
at-most-once; the ``exactly'' is supplied by pairing with the liveness
obligation of Sec.~\ref{sec:formal}). EO constrains observable
external effects only, never message delivery: an effect that commits
while its acknowledgment is lost counts as fired, and the retry
discipline for lost acknowledgments is the idempotency-key composition of
Remark~\ref{rem:fork}.
\end{property}
\begin{property}[FD: Fork determinism]\label{p:fd}
If resumes carrying \emph{fork intent} (Definition~\ref{def:fork}) with
values $v_1,\dots,v_m$ are addressed to the same interrupt checkpoint,
then each branch outcome satisfies $o_k = f(v_k)$ for the branch semantics
$f$; in particular $v_j \ne v_1 \Rightarrow o_j \ne o_1$ whenever $f$ is
injective. Here $f$ is the \emph{decision function}: the framework's
routing of the supplied value into the gated branch decision, deterministic
by construction of the gate; model or tool nondeterminism downstream of the
decision is outside $f$, so FD is well-defined for nondeterministic
agents.
\end{property}
\begin{property}[CV: Checkpoint validity]\label{p:cv}
Every persisted checkpoint record satisfies the state schema: a write that
would persist schema-invalid state is rejected with an error, not stored.
\end{property}
\begin{property}[CO: Consume-once]\label{p:co}
Two clauses, named here because the rest of the paper quantifies over
them separately. \textbf{(CO-c, consumption count)} An interrupt is
consumed by at most one resume. \textbf{(CO-e, effect inertness)} A
resume \emph{without} fork intent addressed to a completed run or an
already-consumed interrupt --- including byte-identical re-delivery of a
prior resume --- is inert with respect to effects.
\end{property}
\noindent The clauses are not interchangeable, and the distinction is
load-bearing twice below. CO-e is violated when a stray delivery
\emph{fires} something; CO-c is violated when authority is
\emph{taken} twice, whether or not a second effect follows --- a gate
that serves its effect idempotently from the durable record can consume
one human approval twice while the effect count stays at one, which
leaves the approval trail wrong and the effect ledger right. The
cross-process failure this paper measures (probe~159) is a CO-c failure
that happens also to break CO-e; the model of Sec.~\ref{sec:formal}
formalizes CO-e only, and Sec.~\ref{sec:independence} reports what each
clause is independent of.
\begin{property}[RD: Recovery determinism]\label{p:rd}
The recovery decision (which tasks to skip versus re-execute) is a
function of durable state: two recoveries from identical durable logs make
identical decisions.
\end{property}

\begin{remark}\label{rem:rd}
CO is the effect-restricted corollary of EO at the gated task, named
separately because the interrupt lifecycle is where deployed systems
concentrate human authority, approval trails are audit subjects, and a
framework can satisfy CO while failing EO elsewhere, as the conformance
matrix shows. (The dependence analysis is
Section~\ref{sec:independence}, which also states the set's
non-completeness and names candidate seventh properties; we flag the
relationship here so no independence assumption is carried for three
pages.) RD is motivated by an ordering \emph{hazard} we reported in
\LG{}'s synchronous durability mode (\#8039): a completed task's pending
writes and the superseding checkpoint go to a shared executor with no
ordering barrier, so a crash between them can leave either of two
durable states. Two obligations travel under RD's name.
Property~\ref{p:rd} as stated is determinism proper --- identical
durable logs, identical decisions --- which the model checks
(Listing~\ref{lst:inv}) and probes~118 and 128 test directly, each
recovering \emph{twice} from one committed crashed log
(\texttt{sha256}-verified byte-identical copies on the durable backend),
decisions identical in every pair. The \#8039 hazard additionally
demands \emph{order-invariance}: a crash inside the unbarriered window
is compatible with more than one legal log, and a recovery deterministic
in each yet divergent across the two leaves post-crash behavior
unreproducible. Probes 118, 128, and 136 test that stronger obligation
on the realized crash pair; the Verus recovery core
(Sec.~\ref{sec:remit}) covers the \emph{completed} window only, not
the crash-truncated pair --- which is how \LG{}'s measured invariance
coheres with its crash-path EO violation. RD's inclusion predates any
observed failure.
\end{remark}

\subsection{Fork intent: making FD and CO jointly satisfiable}
\label{sec:forkintent}
FD and CO constrain behavior at the same wire point --- a second resume
addressed to a consumed interrupt --- in opposite directions: FD demands
the new value be honored on a new branch, CO demands a stray re-delivery
be inert. Without a discriminator the two are jointly unsatisfiable on
identical traffic, so the contract does not leave ``explicit fork''
undefined:

\begin{definition}[Explicit fork]\label{def:fork}
A resume carries \emph{fork intent} iff it bears a branch discriminator
distinguishing it from re-delivery of a prior resume: a distinct resume
ordinal, an explicit fork flag, or an address the framework's own
documentation designates as branch-creating.
\end{definition}

\begin{property}[FI: Fork-intent expressibility]\label{p:fi}
The resume API must make the discriminator of Definition~\ref{def:fork}
expressible on the wire. FI is a protocol obligation on the interface
rather than a behavioral property of a trace; an API that cannot express
fork intent forces implementations to resolve retry-versus-fork by
guessing, and cannot satisfy FD and CO simultaneously. FI is numbered
with the properties because the conformance report needs a column for
it, and it is deliberately not model-checked: it constrains what the
API can say, not what traces do.
\end{property}

\begin{definition}[Intent-indexed FD and CO]\label{def:intentpred}
Let an \emph{intent labeling} assign to each resume $w$ an exogenous
label $\iota(w) \in \{\textsc{fork}, \textsc{retry}\}$ --- the
caller's intent, not a field of $w$. $\mathrm{FD}^{\iota}$ holds of a
responder iff every \textsc{fork}-labeled $w$ is honored on a fresh
branch whose gated effect fires exactly once; $\mathrm{CO}^{\iota}$
holds iff every \textsc{retry}-labeled $w$ is effect-inert. Where the
wire carries a discriminator, $\iota$ is recoverable and these coincide
with FD and CO; where it does not, they are strictly stronger. The
matrix of Sec.~\ref{sec:pilot} reports FD and CO only: no probe can
supply $\iota$ either.
\end{definition}

\begin{lemma}[FD--CO incompatibility without a discriminator]
\label{prop:fdco}
On a wire protocol whose resumes carry no branch discriminator --- none
in the accepted traffic, transport metadata included; Remark~\ref{rem:tworeadings}
delimits the scope --- no
responder --- deterministic or randomized --- satisfies both
$\mathrm{FD}^{\iota}$ and $\mathrm{CO}^{\iota}$
(Definition~\ref{def:intentpred}) on all traffic
the protocol admits, for any intent labeling $\iota$ under which both
labels are realized.
\end{lemma}
\begin{proof}
Let interrupt $c$ be consumed with recorded value $v_1$, branch outcome
$o_1=f(v_1)$, the branch's effect fired once. Construct two executions
identical in durable state, wire trace, local schedule, and environment
inputs (clocks included) up to and including the arrival of $w = \mathit{resume}(v_1)$ addressed to $c$: in
(a), $w$ is a transport duplicate, for which CO requires
effect-inertness; in (b), $w$ is a caller-requested fork re-answering
$v_1$, for which FD with per-branch EO requires a fresh branch whose
gated effect fires exactly once. Every input the responder can condition
on is identical across (a) and (b); only the caller's intent, which by
hypothesis the wire cannot carry, differs. A deterministic responder
produces one behavior for both and violates one requirement. A
randomized responder induces one distribution for both; the two required
behaviors are disjoint events, so its success probabilities sum to at
most one and on at least one intent it errs with probability
$\ge 1/2$. Contradiction. (The construction is an indistinguishability
argument: the two executions are observationally equivalent to the
responder, differing only in an exogenous label the wire cannot carry.)
\end{proof}

\begin{remark}\label{rem:tworeadings}
Scope. Transport metadata is inside the construction: any request id,
retry token, or idempotency key the accepted traffic exposes \emph{is}
a discriminator in the sense of Definition~\ref{def:fork}, making FI
hold and the lemma inapplicable; the impossibility concerns protocols
whose accepted traffic carries no such bit --- the probed \LG{}
surface. And the construction presupposes both intents are
\emph{admissible}: an environment that excludes duplicates or forbids
same-value forks escapes by emptying one side --- a fact about the
environment, not the responder; on the probed local APIs both arrive as
caller invocations.
\end{remark}

Lemma~\ref{prop:fdco} is deliberately \emph{not} an explanation of
\LG{}'s defect. Its role is prior: it is FI's necessity in theorem
form --- demanding FD without demanding FI would be incoherent. \LG{}
is the harder second case: its documentation gives the
explicit-checkpoint address \emph{two} readings --- branch-creating
time travel and replay (Sec.~\ref{sec:plane}) --- so the address meets
Definition~\ref{def:fork} clause~3 on the first reading and fails to
discriminate on the second; FI holds on at most one documented reading,
the dual documentation is itself an instance of the gap FI names, and
the loop serves the recorded value under either reading (FD fails) ---
a factorization that still localizes the defect to the serving logic,
exactly the seam probe~134's repair exploits. Sufficiency is claimed
only constructively: with a discriminator present, a responder keying
branches on it satisfies FD jointly with same-value replay idempotence
--- machine-checked at model level (LGF-B, Table~\ref{tab:tlc}); LGF is
an outcome-serving model with no effect state, so CO's effect half is
exercised instead by \sys{}'s per-branch ledger
(Sec.~\ref{sec:package}).

FI is what \sys{}'s $\langle\mathit{checkpointId},
\mathit{resumeIndex}\rangle$ keying supplies and what the probed \LG{}
protocol lacks: neither \texttt{Command(resume=v)} nor the resume-map
form carries a retry/fork bit, and the address the probe uses --- a prior
\texttt{checkpoint\_id} --- is the one the persistence documentation
presents as the branch-creating time-travel
primitive~\cite{lg-persist-docs}. The \no{} verdict in
Section~\ref{sec:pilot} therefore stands under either documented reading,
while FI names the protocol gap that made the ambiguity possible.

\subsection{Partial independence (interface-relative) and empirical
necessity}\label{sec:independence}
Whether six is the right number is answered by a per-invariant fault
matrix --- each \texttt{ResumeContract} fault model checked against
every property, plus the state-rebuild module; 39 TLC runs,
Table~\ref{tab:ix} --- whose clean rows are precisely the separating
models a logical-independence argument requires, completed by a
companion module supplying the three witnesses the deployed-mechanism
switches entangle (21 further runs, \texttt{R10\_Separations.tla},
every verdict, state count, and counterexample depth identical across
the two environments). The 39-cell matrix is re-derived at the R8
constants with the same verdict in every cell (receipts in
\texttt{results/\allowbreak tla/\allowbreak independence\_r8/}), so
the separations are not an artifact of the small configuration:
counterexamples deepen as the space grows --- the fork fault from
depth~5 to~8, double consumption from~6 to~13 (single-worker minimal
depths, receipts in
\texttt{results/\allowbreak tla/\allowbreak singleworker\_20260806/};
the matrix files log 16-worker abort-time traces, which can run
deeper) --- while every clean
cell stays clean over a state space four orders of magnitude larger. Three scope words govern the title: the separations are
\emph{interface-relative} --- established over the vocabulary of
Definition~\ref{def:plane}, the domain over which the contract itself is
stated, and the relativity every independence claim has to its signature
--- \emph{witness-complete} (each witness fully
explored at both bound sets, every holding cell at R8 exhausting
$1.0$--$1.4{\times}10^{6}$ distinct states against the reference run's
$59$--$396$), and \emph{partial}, because one dependence (CO on EO) is
structural.

One distinction carries the proof-theoretic weight and is stated once,
here. A separation claim is the non-implication $S \nvdash P$, and
non-entailment is proved, in any logic, by exhibiting one model of
$S \wedge \neg P$: a structure over the contract's interface in which
the conjunction holds and the target fails. A single fully checked
finite witness is such a model and settles the claim outright. TLC's
exhaustive breadth-first enumeration of a witness's reachable space is a
complete verification of that witness --- there is no ``beyond the
bound'' for the structure its constants define --- so the
non-implications of Proposition~\ref{prop:independence} are established
by exhibition, never sampled, and the R8 re-derivation is corroboration
that the witnesses are not degenerate configurations, not the source of
their validity. What is genuinely bound-relative is the universal side:
that the reference semantics satisfies the conjunction (R0/R8) and that
each fault's discovered footprint is complete are claims over all
behaviors, exhaustive at the stated constants and silent beyond them.
TLAPS now discharges the first unbounded
(Section~\ref{sec:threats}); footprint completeness retains exactly
this bound-relative status, and no TLAPS step is needed to make the
exhibited witnesses proofs.

\begin{proposition}[Partial independence, machine-checked]
\label{prop:independence}
Over the interface of Definition~\ref{def:plane}, as formalized in
Section~\ref{sec:formal}: (i)~FD is
logically independent of the conjunction of the other five properties:
the \texttt{ForkIgnore} model satisfies EO, PC, CV, CO, and RD over its
entire reachable state space at the reference bounds (TLC complete, no
error) while violating FD.
(ii)~CV is likewise independent, witnessed by \texttt{InvalidPersist} ---
a structurally easy separation, and deliberately so: no other invariant
reads the validity bit because the contract states CV over the write,
so its in-model independence is definitional, while the coupling that
matters in deployment --- invalid state consumed downstream --- is an
empirical question the downstream-consequence probe answers (probe~150,
Sec.~\ref{sec:pilot}).
(iii)~PC is implied by no single property in the set, on non-vacuous
witnesses: EO does not imply PC --- the state-rebuild model (R7)
satisfies EO over its entire state space while violating PC --- and none
of FD, CV, CO, or RD implies PC --- the \texttt{PrefixReplay} model of
(iv) violates PC while all four hold, non-vacuously, over its entire
state space. R7's remaining cells are vacuous (machinery absent);
vacuous cells separate nothing and Table~\ref{tab:ix} labels them as
such. (iv)~CO \emph{splits}, and only one clause is dependent.
\textbf{CO-e is not independent of EO}: as formalized in
\texttt{ResumeContract.tla} the invariant is
$\mathit{effects}[\mathrm{IP}] \le 1$, which is EO restricted to the
gated task, so $\mathrm{EO} \Rightarrow \mathrm{CO\text{-}e}$ and every
CO-e violation is an EO violation --- the matrix confirms this, and we
state plainly that this implication is a fact about the two formulas
rather than a discovery: it would hold whatever the transition relation
did. \textbf{CO-c is independent of the conjunction of all six
others}, including CO-e. The witness is a lost-update model
(\texttt{R11\_ConsumeCount.tla}, receipts in
\texttt{formal/tla/consumecount/}) that carries a consumption counter and
admits a second racer which reads the parked interrupt's waiting flag
before the first racer's clearing write lands. Its reference cell
reproduces R0 exactly (87 generated / 59 distinct); conservativity
holds by construction --- the racer disabled, the counter a function of
the existing trace, the action alphabet R0's --- with the 87/59 identity
as that construction's executable check; with the race enabled and the
gate serving its
effect idempotently from the durable record, TLC completes over 127
generated / 95 distinct states with TypeOK, EO, PC, FD, CV, CO-e, and RD
holding and CO-c violated at depth~5, and the same verdict pattern
re-derives at wider bounds (8{,}191 distinct states, counterexample
depth~6). This matters beyond bookkeeping: CO-c is the clause the
measured cross-process failure breaks, and until it is carried as an
invariant the model cannot express the paper's own headline composition
result. The CO-e implication is
one-directional, with an in-model separating witness for the converse:
the \texttt{PrefixReplay} model --- recovery restarts from task~1 while
the gated task's effect is served from the durable record, the
memoized-gate discipline of \LG{}'s measured crash path --- violates EO
(at a non-gated task) and PC while FD, CV, CO, and RD hold over its
entire reachable state space (TLC complete, 287 generated / 183 distinct
states); the gate-exercising crash-path trace itself is witnessed by the
shipped R9W configuration (consume; crash at frontier IP; replay; gate
served from the record). (v)~RD, PC, and EO are each independent of the
conjunction of the other five, witnessed by a companion module
(\texttt{R10\_Separations.tla}) whose three switches are effect-safe or
control-safe variants of mechanisms this study observed, each
constructed so that exactly one property fails while the other five
hold over the entire reachable state space (TLC complete, no error, at
the reference bounds): \texttt{Regate} --- at equal durable state one
recovery continues past the consumed gate and another re-arms it, the
re-consumption served from the durable record, so the recovery decision
differs (RD \no{}) with no effect fired twice and no completed task
re-executed; \texttt{Rebuild} --- deterministic restart with the
working state taken from initial values rather than re-derived from the
log, every prefix effect served from the record (PC \no{}); and
\texttt{Redeliver} --- at-least-once re-issue of the durable frontier
task's effect while control resumes past it and the gated task is
excluded (EO \no{}). Within the original switch set none of the three
separates, because nondeterministic recovery's footprint is
$\{$EO, PC, CO, RD$\}$ --- its replay branch re-executes the prefix ---
and that entanglement is itself the finding recorded below. In sum:
independence from the \emph{conjunction} of the other properties is
established for EO, PC, FD, CV, RD, and CO-c; CO-e alone is
definitionally dependent, by (iv).
\end{proposition}

Two comments. CO is retained as a named property despite (iv) because it
isolates a distinct production mechanism --- stray duplicate delivery
versus crash replay --- and the framework data exhibit the same one-way
structure: \LG{}~1.2.9 fails EO on the crash path while CO holds on the
same configuration (Table~\ref{tab:matrix}). And the entanglement (v)
records is itself a finding: recovery indeterminism whose branch
includes replay drags replay's damage along --- exactly what the \CA{}
restore receipt exhibits live --- which is why the witnesses are
effect-safe variants: they show logical independence, not that any
framework fails one property at a time.

The empirical side carries its own weight. For each of EO, FD, CV, CO,
and the liveness obligation there is an observed framework \emph{path}
on which that property fails while the others hold on that path; every
observed PC violation (\CA{}'s rebuild-from-initial restore) co-occurs
with an EO violation, so at framework level PC's necessity is claimed
jointly, with the model-level separation in
Proposition~\ref{prop:independence}(iii). Three claims, kept distinct:
(i)~empirical necessity, per the observed paths above; (ii)~formal
partial independence --- Proposition~\ref{prop:independence};
(iii)~no theorem of minimality, completeness, or sufficiency is claimed.
The set is not reverse-engineered from the failure list: RD entered from
the recovery question with no observed failure then or now. CO's statement is
likewise unchanged from the pre-measurement set; what measurement
sharpened is which path evidences it: the sequential stray
resume that motivated the property is refused by the probed plane, while
the same delivery from a second process is admitted (probe~159), so the
property earns its empirical necessity on the concurrent path.

\emph{Why six --- and could there be a seventh?} One property per
question a caller must be able to answer across the boundary, with FI
the protocol precondition the fork question needs. Merging is blocked by the separations above; splitting further is
unmotivated by any observed mechanism.
Candidate sevenths are named as scope exclusions: cross-version
checkpoint-migration validity, an effect-visibility ordering obligation,
and loud-versus-silent disposition of discarded resumes
(Section~\ref{sec:pilot} flags the last as a spec-level gap).

\emph{Scope and composition.} Definition~\ref{def:plane} linearizes: the
properties are stated per causal chain, and a DAG's parallel branches are
each such a chain; sequential composition across gates and concurrent
fan-out at one superstep are both probed (Sec.~\ref{sec:pilot},
probes~138, 141), while deeper nesting and multi-checkpointer fan-in
remain future work. Invisible speculative execution that never surfaces
through the public state API is outside the contract. CV is per schema
version, and checkpoint \emph{migration} is a named exclusion: records
valid under schema $v$ are not CV subjects under $v{+}1$, so evolution
requires an explicit migration step before any resume --- a plane that
silently auto-migrates reintroduces exactly the hazard CV names
(\sys{}'s gate answers relative to the validator supplied at load). CO's
resume identity is the checkpoint address plus payload; transport
metadata is excluded. The properties are not claimed orthogonal --- a CV
failure undermines the addressability CO depends on --- so conformance
is reported per-property, never as a single bit.

\begin{remark}[What the conjunction buys, informally]
Under Definition~\ref{def:plane}: EO and CO fix effect multiplicity, PC
the continuation state, RD the recovery choice, CV log integrity, and
FD with FI the branch outcomes --- jointly, the caller-visible outcome is
a function of the inputs, the durable log, and the supplied decisions.
Orientation, not a theorem; R0/R8 machine-check the conjunction over the
model.
\end{remark}

\begin{remark}[Fork and exactly-once]\label{rem:fork}
EO is per-branch by design: a fork is user-requested duplication of
downstream execution, so the gated effect legitimately fires once per
requested branch. The contract forbids duplication nobody asked for ---
re-execution within a branch across crash or resume (EO), and effects
fired by resumes that requested nothing (CO). A deployment for which even
requested branch effects are unacceptable composes FD with an idempotency
key at the effect itself; the contract governs the plane, not the payment
processor.
\end{remark}

\begin{remark}[Relation to classical semantics]
EO restricted to crash recovery is the workflow analogue of exactly-once
\emph{processing} (state-management, not delivery~\cite{flink};
throughout, EO is an effect guarantee, never message delivery); EO/CO
restate the application-level idempotency discipline Helland argues is
unavoidable~\cite{helland-cidr,helland-idem}; PC's memoized-replay clause
is deliberately the discipline formalized for Durable
Functions~\cite{durablefn} and practiced by event-sourced
actors~\cite{akka}, so replay-based recovery \emph{conforms} to PC while
\CA{}'s rebuild-from-initial restore does not; RD is determinism of the
recovery function in the sense of rollback recovery~\cite{elnozahy}. The
contribution is not these notions but their assembly into a checkable
obligation set for an API where today they are unstated, divergent, and
violated.
\end{remark}

\section{Machine-Checked Model}\label{sec:formal}
The object of verification throughout this section is the
\emph{specification}, not any framework: TLC exhausts the reference
semantics and its fault variants, while the \LG{}-derived module of
Section~\ref{sec:derived} is a separately flagged, expert-established
abstraction. The model inherits the contract's scope exclusions
(Section~\ref{sec:contract}).

\subsection{The \texorpdfstring{TLA\textsuperscript{+}}{TLA+} module}
We formalize Definition~\ref{def:plane} and
Properties~\ref{p:pc}--\ref{p:rd} as a TLA\textsuperscript{+}
module, \texttt{ResumeContract.tla} (251 lines, in the artifact), paired
with a consumption-counting companion, \texttt{R11\_ConsumeCount.tla},
that carries the CO-c invariant the base module's atomic
\texttt{Consume} cannot express (the pairing is motivated where that
gap is stated, below). State
comprises the program counter, per-task effect counters, the checkpoint
log, the durable frontier, the interrupt flag, the consumed value, the
sequences of fork values and fork outcomes, the crash count, the recovery
history (pairs of durable frontier and decision), and a prefix-regression
flag. Actions are \texttt{ExecTask}, \texttt{EmitInterrupt},
\texttt{Consume(v)}, \texttt{ForkResume(v)}, \texttt{CrashRecover}, and
\texttt{ExtraResume}. The six properties are state invariants; the
listing shows four of them verbatim.

\begin{lstlisting}[caption={Contract invariants in
\texttt{ResumeContract.tla} (excerpt).},label={lst:inv}]
EffectExactlyOnce ==
  \A t \in Tasks : effects[t] <= 1
ForkDeterminism ==
  \A k \in 1..Len(forkOuts) :
    forkOuts[k] = f(forkVals[k])
CheckpointValidity ==
  \A k \in 1..Len(ckpts) : ckpts[k].valid
RecoveryDeterminism ==
  \A i, j \in 1..Len(recHist) :
    recHist[i].dur = recHist[j].dur
      => recHist[i].dec = recHist[j].dec
\end{lstlisting}

Six Boolean fault switches select the reference semantics or one observed
violation class --- each switch transcribes a deployed mechanism, not a
negation of a target property, and Table~\ref{tab:ix} then measures what
each mechanism actually breaks: \texttt{FaultReplay} (recovery restarts from task~1 over
restored state --- the \CA{} restore discipline, and the \#7361 regression
class), \texttt{FaultForkIgnore} (a second resume at the same checkpoint
is answered with the first branch's outcome --- \#6663),
\texttt{FaultInvalidPersist} (a schema-invalid completion record is
persisted silently --- the \#6491 class), \texttt{FaultNondetRecovery}
(skip-versus-re-execute unconstrained at equal durable state --- \#8039),
\texttt{FaultDoubleConsume} (a stray resume on a completed run
re-fires the gated effect --- the \#2315 class~\cite{copilotkit2315}),
and \texttt{FaultPrefixReplay} (recovery restarts from task~1 while the
gated task's effect is served from the durable record --- memoized-gate
prefix replay, the \LG{}~1.2.9 crash-path class of probes~118/133;
added after those measurements). A shipped witness configuration (R9W) checks the negation
of the gate-exercised state, with a \texttt{recHist} conjunct forcing
consumption \emph{before} the crash; TLC's depth-7 counterexample is the
memoized-gate crash path realized end to end --- the interrupt is consumed
(frontier reaches IP), the crash recovers from \texttt{dur = IP}, prefix
replay re-executes task~1 (the EO violation), and the replayed pass crosses
the gate with the gated effect served from the durable record,
\texttt{effects[IP]} unchanged on the trace.

\subsection{TLC results}
Table~\ref{tab:tlc} reports the verification matrix (TLC~2026.04.09,
rev.~389440f, one
worker, $N{=}3$ tasks, $\mathit{IP}{=}2$, $|V|{=}2$, at most two resumes,
one crash --- two for the RD run --- and one stray resume). The bound is
a small-scope argument: every property quantifies over per-task phases
relative to the interrupt point and over recorded-versus-supplied values,
so three tasks with $\mathit{IP}{=}2$ realize every phase relationship
the transition relation admits, and two values every equality pattern; each violating run finds its
counterexample at depth $\le 7$ at these bounds; and R8 re-checks the same invariants at
ten tasks over $7.4{\times}10^6$ distinct states, unchanged.
Single-worker BFS makes every counterexample minimal-depth and every
number bit-reproducible; \texttt{reproduce.sh} re-derives the rows. The
reference configuration satisfies all six invariants; each single-fault
configuration violates its target at depth 4--6 (R1--R5); the
per-invariant matrix of Table~\ref{tab:ix} then reports each fault's
\emph{full} violation footprint.

\begin{table}[t]
\caption{TLC verification matrix for \texttt{ResumeContract.tla}
(TLC~2026.04.09, rev.~389440f; every fault row single worker, so
counterexample traces are BFS-minimal and deterministic across hosts
(R8-F receipt:
\texttt{results/\allowbreak tla/\allowbreak singleworker\_20260806/});
the clean scaled row R8 ran 4 workers, which cannot alter a completed
exhaustive search's totals; states = generated/distinct; CE =
counterexample depth). R0 checks all six invariants;
R1--R5 and R9 check the targeted invariant under a single fault switch
(full per-fault violation footprints: Table~\ref{tab:ix}); R6 checks the liveness
obligation under weak fairness; LGF-A/LGF-B check the framework-derived
model of Section~\ref{sec:derived} with no fault switches; R9W's
counterexample is the by-design gate witness.}
\label{tab:tlc}
\centering
\scriptsize
\setlength{\tabcolsep}{3pt}
\begin{tabularx}{\columnwidth}{@{}l>{\raggedright\arraybackslash}Xllr@{}}
\toprule
Run & Fault switch & Invariant & Result & States; CE \\
\midrule
R0 & (reference) & all six & \yes{} no error & 87/59; --- \\
R1 & \texttt{Replay} & EO & \no{} violated & 7/7; 4 \\
R2 & \texttt{ForkIgnore} & FD & \no{} violated & 10/10; 5 \\
R3 & \texttt{InvalidPersist} & CV & \no{} violated & 8/8; 5 \\
R4 & \texttt{NondetRecovery} & RD & \no{} violated & 10/10; 4 \\
R5 & \texttt{DoubleConsume} & CO & \no{} violated & 20/20; 6 \\
R9 & \texttt{PrefixReplay} & EO & \no{} violated & 7/7; 4 \\
R9W & \texttt{PrefixReplay} & gate witness & witness CE & 45/35; 7 \\
\addlinespace[2pt]
R6 & (reference) & liveness & \yes{} no error & 87/59; --- \\
\addlinespace[2pt]
LGF-A & (as implemented) & FD & \no{} violated & 5/5; 3 \\
LGF-B & (fork-keyed) & FD + idem. & \yes{} no error & 7/7; --- \\
\addlinespace[2pt]
R8 & (ref., scaled) & all six & \yes{} no err. & 14.7M/7.4M; --- \\
R8-F & \texttt{ForkIgnore} (sc.) & FD & \no{} violated & 68/68; 8 \\
\bottomrule
\end{tabularx}
\end{table}

\begin{table}[t]
\caption{Per-invariant fault matrix: each single-fault model checked
against each property individually (one TLC run per cell; 36 cells on
\texttt{ResumeContract.tla} with the corresponding R1--R5/R9 constants,
plus the R7 state-rebuild module row --- EO, PC, and the type invariant
checked; the fork/validity/consume/recovery machinery is absent, so those
cells are structurally vacuous --- 39 runs total, receipts in
\texttt{formal/tla/independence/}). \no{} = violated (counterexample
depth in parentheses where recorded); \yes{} = the invariant holds over the faulty
model's \emph{entire} reachable state space (TLC complete, no error).
Rows whose off-target cells are all \yes{} are separating models for
their targeted property; off-diagonal \no{} cells are each fault's
discovered footprint.}
\label{tab:ix}
\centering
\scriptsize
\setlength{\tabcolsep}{4.5pt}
\begin{tabular}{@{}lcccccc@{}}
\toprule
Fault model & EO & PC & FD & CV & CO & RD \\
\midrule
\texttt{Replay} (R1 consts) & \no{}\,(4) & \no{}\,(4) & \yes{} & \yes{} & \no{}\,(7) & \yes{} \\
\texttt{ForkIgnore} (R2) & \yes{} & \yes{} & \no{}\,(5) & \yes{} & \yes{} & \yes{} \\
\texttt{InvalidPersist} (R3) & \yes{} & \yes{} & \yes{} & \no{}\,(5) & \yes{} & \yes{} \\
\texttt{NondetRecovery} (R4) & \no{} & \no{} & \yes{} & \yes{} & \no{} & \no{}\,(4) \\
\texttt{DoubleConsume} (R5) & \no{}\,(6) & \yes{} & \yes{} & \yes{} & \no{}\,(6) & \yes{} \\
\texttt{PrefixReplay} (R9) & \no{}\,(4) & \no{}\,(4) & \yes{} & \yes{} & \yes{} & \yes{} \\
\texttt{StateRebuild} (R7 mod.) & \yes{} & \no{}\,(3) & \multicolumn{4}{c}{(vacuous: machinery absent)} \\
\bottomrule
\end{tabular}
\end{table}

Table~\ref{tab:ix} answers two objections at once. To the charge that
the fault switches are circular --- engineered to violate their
properties --- the matrix replies that a switch injects a
\emph{mechanism} and TLC discovers its footprint: replay breaks EO, PC,
\emph{and} CO; injected recovery nondeterminism drags replay's damage
plus RD; only the fork and validity mechanisms break exactly one
property each. If the switches merely stipulated their targets, every
off-diagonal cell would be clean; four rows are not. To the demand for
independence, the clean rows deliver it: \texttt{ForkIgnore} and
\texttt{InvalidPersist} are full separating models
(Proposition~\ref{prop:independence}), the state-rebuild row separates
PC from EO, \texttt{PrefixReplay} separates EO from CO in the only
possible direction, and the one pair that cannot come apart is reported
as the structural implication it is. The full matrix replicated across
hosts with identical verdicts, depths, and state counts. The invariants
read as one-liners by design; the proof burden lives in the six actions
whose every interleaving must preserve them, in
\texttt{ResumeContract.tla}.

The liveness row closes a vacuity hole: a framework that refuses every
resume satisfies all six safety invariants, so the contract pairs them
with \texttt{EventuallyCompletes} under weak fairness, verified on the
reference configuration (and violated in spirit by a deployed framework
whose crash leaves persistence unrestorable, Section~\ref{sec:pg}).

The R1 counterexample is worth reading against the live data: execute
task~1 ($\langle1,0,0\rangle$, frontier~1 durable); crash and replay;
execute task~1 again ($\langle2,0,0\rangle$) --- EO violated in four
states. Section~\ref{sec:pilot} measures the same arithmetic live in
\CA{}: task~1 persisted, crash in task~2, resume, and the completed
task re-fires, landing the counter on 12 where exactly-once predicts
11.

\emph{Modeling assumptions and what the checking means.} Scheduler state
and message plumbing are abstracted --- the abstraction the properties
themselves quantify over. The fault switches inject mechanisms observed
in deployed frameworks and Table~\ref{tab:ix} discovers what each
breaks; none of this shows the specification captures any framework, a
burden that falls on Section~\ref{sec:derived} and the measurements.
The configuration is a dial, not a ceiling: the reference space is the
minimal one in which every property is falsifiable, while R8 (10 tasks,
interrupt at 5, 4 values, 5 resumes, 4 crashes, 4 extra resumes) passes
the six invariants plus \texttt{TypeOK} over $1.47{\times}10^{7}$
generated / $7.4{\times}10^{6}$ distinct states at depth 24 --- an
exhaustive breadth-first check of the full reachable space at those
constants --- and the fork fault still yields a counterexample at
depth~8 under single-worker search. (Workers race to counterexamples on
scaled fault runs, so only the violated invariant is a stable receipt
there; every depth reported here is a single-worker depth.)

Three scope bounds of the model are stated rather than discovered.
\emph{First}, the modeled crash is a recovery-decision event:
\texttt{CrashRecover} carries the precondition that no interrupt is
pending, so a crash while the run is parked awaiting the human is outside
\emph{this} module's transition relation. That case is covered
empirically by probes~158/158b/158c on all three planes that park, and
formally by a companion module, \texttt{ResumeContractParked.tla}: under
durable parking --- the measured behavior of every plane that parks ---
all six invariants and the liveness obligation hold at both bound sets
and the fault matrix re-derives with verdicts and depths identical to
Table~\ref{tab:ix}; under volatile parking the six safety invariants
hold while \texttt{EventuallyCompletes} is violated --- safety by
deadness, pydantic-graph's disposition (Sec.~\ref{sec:pg}), now a
model-level counterexample and not only a measured cell (receipts in
\texttt{formal/tla/parked/}, both environments).

\emph{Second}, fork branches are outcome-level: \texttt{ForkResume}
records values and outcomes and fires no effects, so Remark~\ref{rem:fork}'s
per-branch accounting is a property of \sys{}'s ledger
(Sec.~\ref{sec:package}), not of this module. \emph{Third}, and the one
that matters, CO-c is \emph{not representable here at all}:
\texttt{Consume(v)} clears \texttt{waiting} atomically, so a second
consumption of a live parked interrupt is not a behavior this transition
relation admits, and the module carries no consumption counter. That is a
real gap, not an abstraction choice, because the cross-process failure of
Sec.~\ref{sec:pilot} is exactly that second consumption; and
\texttt{FaultDoubleConsume} is not its shadow, being guarded by
$\mathit{pc} = \mathit{NTasks}+1$ --- a \emph{completed} run, the
sequential mechanism the probed plane refuses. The companion module
\texttt{R11\_ConsumeCount.tla} closes it: a consumption counter, CO-c and
CO-e as separate invariants, and one extra racer reading the waiting flag
before the first racer's clearing write lands --- probe~159's two-racer
shape, the racer served its own value (FD intact, as measured) and
appending no checkpoint (the duplicate invisible in framework state, as
measured). Its reference cell reproduces R0 exactly (87 generated / 59
distinct), the executable check of a by-construction conservativity
(racer disabled; counter a function of the R0 trace). With the race
enabled and
the gate idempotent, TLC completes over 127 / 95 states with EO, PC, FD,
CV, CO-e, RD and \texttt{TypeOK} holding and CO-c violated at depth~5 ---
the separating witness of Proposition~\ref{prop:independence}(iv). With
the race also firing the effect, the discovered footprint is
$\{\mathrm{EO}, \mathrm{CO\text{-}e}, \mathrm{CO\text{-}c}\}$ with PC, FD,
CV, RD clean, matching probe~159 cell for cell. Both patterns re-derive
at wider bounds (8{,}191 distinct states, depth 6). The module's 32-cell
matrix replicates across two separately provisioned environments and across a
JDK major version, so the counts are a property of the module and the
checker rather than of a toolchain.

Liveness uses weak fairness only --- no action is repeatedly disabled and
re-enabled in competition, so strong fairness would add assumptions
without theorems --- and is a progress check, not an availability claim.
Division of labor: TLC checks the \emph{protocol}; Verus
(Sec.~\ref{sec:remit}) checks \sys{}'s invariants --- never validation of
the frameworks.

\subsection{A framework-derived model: the fork violation is the shadow
of replay idempotence}\label{sec:derived}
The fault-switch runs demonstrate that the contract is precise enough to
check; they say nothing about frameworks. A second module,
\texttt{LangGraphFork.tla}, closes that gap for the fork axis. It models
\LG{}'s resume path as an operational abstraction consistent with the
persistence documentation and the reproduced behavior of \#6663: a resume
value delivered to a (thread, checkpoint) is recorded as a pending write,
and an invocation addressed to a checkpoint already carrying a resume
write replays the recorded write rather than recording a new one. No
fault switch exists in the module. TLC \emph{discovers} the violation:
with two invocations supplying $\langle v_a, v_b\rangle$, the served
outcomes are $\langle v_a, v_a\rangle$ --- FD violated at depth~3 (LGF-A)
--- while \texttt{ReplayIdempotence} holds throughout. Together that is
the finding: the recorded-write rule is exactly what makes same-value
re-invocation idempotent, a designed property, and exactly what breaks
forking --- \#6663 is not a slip but the shadow of a design choice. The
repaired algorithm --- pending writes keyed by (checkpoint, resume
ordinal), precisely \sys{}'s branch keying --- verifies idempotence and
fork determinism together (LGF-B).

The modeled rule is located in source. The dedup guard in
\texttt{InMemorySaver.put\_writes} skips only non-negative write
indices, and the write-index table maps the resume channel to a
\emph{negative} slot (\texttt{\_\_resume\_\_}\,$\mapsto -4$), exempting
it; the durable savers upsert those channels. The second resume value
is therefore \emph{durably recorded} --- under the null task id --- and
never consulted: the binding site is task preparation, where the
scratchpad constructor in \texttt{pregel/\_algo.py} resolves a task's
resume by precedence (task-recorded first; null-task only via
\texttt{get\_null\_resume}; resume-map values appended \emph{after} the
recorded list), so an invocation addressed to a checkpoint whose task
already carries a consumed resume is served the recorded value
regardless of what it supplied, in every invocation form. The probe
data corroborate each clause: the violation is identical under the bare
and resume-map forms on all three backends (probes~126, 127, 130); the
store dumps show both values durably present with the first served
twice; and the fork-keyed saver shim cannot repair it (probe~125) ---
no saver-level keying can override a decision the saver never makes.
The module's recorded-write rule is this precedence at specification
granularity, with the responsible lines quoted in the artifact:

The mapping is tabulated line for line in Supplement~S5, with the
responsible \LG{}~1.2.9 source lines quoted in the artifact. The mapping is expert-established, not
tool-certified
(Section~\ref{sec:threats} scopes it); probe~134
(Section~\ref{sec:remit}) realizes LGF-B's repair operationally --- and
the 125/134 matched pair is \emph{interventional} evidence for the
mechanism account, not a correlation: altering the serve decision at
the located site flips the verdict, while altering persistence below it
does not. It is also mechanically exercised: a trace-conformance checker
(probe~143) transliterates the serve, record, filter, and inertness
rules into $\sim$20 lines and must reproduce the measured branch
outcome of every recorded invocation --- $8/8$ across the four
fork/stray protocols; its first run caught an omitted inertness clause
and forced a model correction. Evidence for the mapping, not a
refinement proof --- and consistency with the cases the model was built
from, not validation. Probe~171 supplies the missing direction: four
protocols the model has never been exercised on, each prediction
registered \emph{in the probe source and committed before the run},
only then measured --- three invocations
$\langle a,b,c\rangle$ at one interrupt checkpoint (predicted
$\langle a,a,a\rangle$); a fork after a post-completion stray
(predicted $\langle a, \bot, a\rangle$, testing that inertness and the
serve rule compose); bare and resume-map forms interleaved (predicted
$\langle a,a,a\rangle$); and, as negative control, two interrupts on
distinct checkpoints (predicted $\langle a,b,a,b\rangle$). The control
carries the weight: a degenerate model serving the first value ever
seen would satisfy the first three and fail the fourth. All four
confirmed. The mapping remains expert-established rather than
mechanically extracted; what changes is that it now makes falsifiable
predictions off the set it was derived from, and has survived them.

\section{Conformance Methodology}\label{sec:method}

\subsection{Design principles}
The harness obeys three constraints that make every verdict deterministic
and portable.
\begin{itemize}[leftmargin=1.2em]
\item \textbf{LLM-free.} All probes drive persistence APIs with
pure-Python tasks and deterministic values; no verdict depends on
sampling, provider state, or prompt behavior.
\item \textbf{Timing-free crash injection.} Every crash verdict is
decided without a timing window: matrix crashes are exception-based;
probe~133 kills the process with \texttt{SIGKILL} synchronized by a
filesystem barrier; probe~137 widens the kill to a three-point matrix;
and probe~160 exhausts the interface --- one run per persistence
operation, the durable prefix frozen at exactly that many completed
operations by a serializing wrapper, the freeze self-audited per point.
No sleep-window races, no kill-time distributions: a crash-conformance
verdict is a protocol property, host-invariant at a fixed dependency
resolution and runtime, which is what the artifact pins. The three
concurrency cells (probes~159, 168, 174) are the \emph{named
exception},
because their subject is a race: their verdicts are existence claims
decided by the external ledger and reported as observed frequencies ---
saturated at $1.0$ in the measured range --- never as timing-free
protocol facts. Confidence intervals appear only for the live cells,
whose repetitions are genuine samples.
\item \textbf{Effect-ledger oracle.} Each side-effecting task increments
a process-local counter; EO/CO read the counters, PC reads execution
order and state lineage, FD compares branch outcomes to supplied values,
CV inspects the durable log after a schema-violating write, and RD
compares recovery decisions across constructed durable states
(probes~118, 128). On the durable-backend probes the counter is
cross-checked by an on-disk SQLite effect ledger, independent of the
checkpointer database and of process memory; the two oracles agree on
every reported cell (probe~126). Both oracles are external to framework
state: as the \CA{} restore receipt shows, framework-visible state can
look correct while the ledger records a duplicate. The ledger needs no two-phase commit with the effect: on the
durable-backend and kill probes the effect \emph{is} the ledger append
--- a single autocommitted \texttt{INSERT}, durable before the task
returns, every kill gated on a durable event after task return --- so
the ledger equals the effects that fired, by ordering, and the only
hypothetical miswiring (a lost record) undercounts, making every
reported duplicate a floor. Probe~163 checks the ordering mechanically:
kills before the \texttt{INSERT}, after it, and after the framework's
durable write yield post-crash counts of 0/1/1, and only the
post-\texttt{INSERT} window duplicates on resume ($1\!\to\!2$, every
repetition, both environments).
\end{itemize}

RD needs more than straight-line protocols: it concerns the ordering of
persistence operations, so its probe \emph{constructs} the two legal
durable states a crash between unordered operations would leave ---
exhaustive exploration via a dropping checkpointer --- rather than racing
a kill. Two interleavings, two deterministic resumes, one comparison.

Every probe writes raw and stable-view JSON into a per-campaign evidence
directory under \texttt{results/} (one subdirectory per campaign) with a
generated
manifest; the paper's numbers are claimed only from committed evidence,
which the artifact's audit re-derives. The live-keyed
ecological checks gate on API keys, are skipped without them, and
their stable fields are effect counters only (Section~\ref{sec:live}).

\subsection{Provenance: what is reproduced versus newly observed}
\label{sec:provenance}
Table~\ref{tab:provenance} records each finding's provenance. Two
results are new observations, in no tracker we found: the path-split of
exactly-once in \LG{} (probe~118's control run) and the executor-layer
submission-order evidence for \#8039 (probe~124) --- that issue is
ours, filed during the control-plane study~\cite{soundgate} and
disclosed there as such, so it is author-reported upstream evidence,
with several other developers' reproductions as the independent half.
The remainder are deterministic reproductions of third-party filed
issues --- itself a contribution, since issues are anecdotes until
reproduced under controlled conditions with an effect oracle --- and
reproduction is never dressed as discovery.

\begin{table}[t]
\caption{Provenance of each finding. Repro = deterministic reproduction of a
filed issue; New = no third-party tracker report of the measured
behavior (\#8039 is in the tracker as our own filing,
Sec.~\ref{sec:provenance}); Stable = holds across the
version sweep (Sec.~\ref{sec:sweep}).}
\label{tab:provenance}
\centering
\footnotesize
\begin{tabularx}{\columnwidth}{@{}llX@{}}
\toprule
Finding & Status & Note \\
\midrule
\#6663 FD & Repro, Stable & 5 versions; source mechanism located \\
\#6491 CV & Repro, Stable & 5 versions; now silent on 1.2.9 \\
EO crash-path split & \textbf{New} & probe~118 control; interrupt vs crash \\
\#8039 executor order & \textbf{New} & probe~124; unbarriered adjacency \\
CrewAI restore replay & Repro & probe~115; 12-vs-11 arithmetic \\
LlamaIndex prefix (D) & Repro & documented at-least-once \\
pydantic-graph unrecov. & \textbf{New} & probe~119; safety-by-deadness \\
\bottomrule
\end{tabularx}
\end{table}

\subsection{Probe suite}
Forty-seven numbered probes in eleven campaigns, plus the key-gated
multi-model, multi-host live-replication matrix (probe~148,
Sec.~\ref{sec:live}); the artifact's manifest is the authoritative
inventory, and the campaign-by-campaign summary is reproduced in the
supplementary material (Sec.~S2).
Every probe emits raw and
stable-view JSON; the replications that count are cross-host
(Section~\ref{sec:threats}) and cross-backend
(Section~\ref{sec:pilot}).

\subsection{Environment and versions}
Pilot environment: Python 3.12.3, Ubuntu~24.04 container; \LG{}~1.2.9 with
\texttt{langgraph-checkpoint}~4.1.1 (\texttt{InMemorySaver});
\texttt{llama-index-workflows}~2.22.2; \CA{}~1.15.2 (JSON checkpoint
provider, SQLite flow persistence); \texttt{pydantic-graph}~1.107.1;
AutoGen AgentChat~0.7.5, all freshly resolved on
2026-07-16 to test current releases and then pinned by the committed
lockfiles. Durable-backend environment
(probes~126--130): a second, separately provisioned Ubuntu~24.04
container (Python~3.12), fresh-resolved on 2026-07-17 to the identical
framework versions, with \texttt{langgraph-checkpoint-sqlite}~3.1.0,
\texttt{langgraph-checkpoint-postgres}~3.1.0 (psycopg~3.3.4), and a live
local PostgreSQL~16 server. Because verdicts are deterministic and
timing-free, the environment's only load-bearing attribute is the package
version set, which the artifact pins.

\section{Conformance Results}\label{sec:pilot}

\begin{table*}[t]
\caption{Conformance matrix: probed releases (pilot 2026-07-16;
durable-backend replication 2026-07-17).
\yes{} = property holds on the probed path, and only on it; for PC, a
\yes{} certifies the observable consequence of Property~\ref{p:pc}
(state equality with the effect record intact), not provenance --- PC
violations are decidable from surface observables, PC satisfaction is
not, and this qualifies every \yes{} in the column. \LG{}'s CO \yes{}
is the sequential path, which probe~159 shows does not compose across
processes;
\no{} = live violation
(deterministically reproducible); \docd{} = documented divergence (the
framework states a discipline weaker than the property; behavior matches
its statement); \und{} = undocumented divergence (measured re-execution
of completed effect-bearing work with no citable stated discipline for
the probed mechanism: behavior reported in full, label withheld);
$\circ$ = not probed (no crash-time durability control on
the probed path, or mechanism absent). ``regr.'' marks properties whose
violation shipped as a 1.1.x regression and is fixed in 1.2.9.}
\label{tab:matrix}
\centering
\footnotesize
\begin{tabular}{@{}lllllll@{}}
\toprule
Framework (path) & PC & EO & FD & CV & CO & RD \\
\midrule
\makecell[l]{\LG{}~1.2.9 (graph, \texttt{interrupt}/\\ \texttt{Command}; \texttt{InMemorySaver})} &
\yes{}\,(regr.\ \#7361) &
\makecell[l]{\yes{}\ int.\,(regr.\ \#6792)\\ \no{}\ crash (118)} &
\no{}\ \#6663 &
\no{}\ \#6491 class & \yes{} & \yes{}\textsuperscript{a} \\
\makecell[l]{\LG{}~1.2.9 (same paths; \texttt{SqliteSaver} /\\ live \texttt{PostgresSaver})} &
\yes{} &
\makecell[l]{\yes{}\ int.\\ \no{}\ crash (126, 130)} &
\makecell[l]{\no{}\ both forms\\ (126, 127, 130)} &
\makecell[l]{\no{}\ silent\\ (126, 130)} &
\makecell[l]{\yes{}\,seq.\\ \no{}\,159\textsuperscript{e}} & \yes{}\textsuperscript{a}\,(128) \\
\makecell[l]{\LI{}~2.22.2 (\texttt{Context} snapshot;\\ two-step HITL)} &
\yes{} & \yes{} & \yes{} & \yes{}\,(fail-fast) & \yes{} & $\circ$ \\
\makecell[l]{\LI{}~2.22.2 (\texttt{wait\_for\_event}\\ durable-HITL idiom)} &
\docd{} & \docd{} & \yes{} & \yes{} & \yes{} & $\circ$ \\
\makecell[l]{AutoGen AgentChat~0.7.5\\ (\texttt{save\_state}/\texttt{load\_state}, 140)} &
$\circ$ & \yes{}\ restore & $\circ$ & \yes{}\,(loud) & \yes{}\textsuperscript{d} & $\circ$ \\
\CA{}~1.15.2 (\texttt{@persist} restore) &
\und{}\textsuperscript{b} & \und{}\textsuperscript{b} &
$\circ$ & $\circ$ & $\circ$ & $\circ$ \\
\makecell[l]{\CA{}~1.15.2 (\texttt{CheckpointConfig}\\ + \texttt{from\_checkpoint})} &
\no{} & \no{} & $\circ$ & $\circ$ & $\circ$ & $\circ$ \\
\makecell[l]{pydantic-graph~1.107.1\\ (\texttt{FileStatePersistence})\textsuperscript{c}} &
\yes{} & \yes{} & $\circ$ & \yes{}\,(fail-fast) & \yes{} & $\circ$ \\
\bottomrule
\end{tabular}

\vspace{3pt}
\noindent\begin{minipage}{\textwidth}\scriptsize
\textsuperscript{a}\,RD
holds at checkpointer-API granularity: probes~118 (in-memory) and 128
(\texttt{SqliteSaver}) construct both legal durable orders of the \#8039
persistence pair and both recover identically on 1.2.9; the
executor-submission layer is instrumented (probe~124) and the pool
divergence is realized deterministically via an adversarial saver
(probe~136), with recovery invariant.

\textsuperscript{b}\,Classification rule, total over the outcome
classes this study observed and symmetric with one deliberate
exception: \no{} requires either a citable stated
discipline the measured behavior contradicts or silent persistence of
corrupt durable state; \docd{} requires a citable statement of the
weaker discipline; absent any citable statement for the probed
mechanism, measured re-execution of completed effect-bearing work is
\und{} --- reported in full, labeled neither violation nor divergence,
reclassifiable in either direction only by a citable statement, which
changes the label and never the measurement. The asymmetry is deliberate: the contract is proposed, not adopted,
so \no{} charges a framework only with contradicting its own citable
semantics or corrupting durable state, while divergence from the
contract is what every cell, \und{} included, measures. \CA{}'s \texttt{@persist}
documentation frames restoration as state restore and states no replay
discipline in either direction, so those cells are \und{}; the
\texttt{CheckpointConfig} cells remain \no{} against the feature's
citable exactly-once claim. Under the strictest alternative rule every \und{} cell becomes
\no{} and no pairwise separation of observation~(i) changes
(Supplement~S4).

\textsuperscript{c}\,pydantic-graph's \yes{} cells are qualified by a recoverability failure:
after a crash, \texttt{iter\_from\_persistence} raises
(\texttt{GraphRuntimeError}, unable to restore) --- effects are not
duplicated because nothing resumes, so safety holds while progress fails
(Section~\ref{sec:pg}); v2.x removes the persistence machinery
entirely.

\textsuperscript{d}\,AutoGen: double-restore of
the same saved state added zero effects (probe~140); an
interrupt-authority axis does not exist on this plane, so CO is reported
as restore-inertness. AutoGen's plane persists nothing framework-side
--- \texttt{save\_state} hands the state to the caller, who owns
durability --- so Property~\ref{p:cv}'s write clause has no
framework-side instance; the probed CV analog is load-time validation
of caller-supplied state, and the \yes{} certifies exactly that (the
tampering is of the caller-held copy, never of anything the framework
persisted).

\textsuperscript{e}\,CO on the durable backends splits by
concurrency: sequential re-delivery is inert (probes~126, 130), while
the same delivery from two concurrent OS processes fires the gated
effect twice, $10/10$ on both durable backends on the developer host
(probe~159; the container replication is shorter and duplicates in every
repetition) --- the \emph{default} packaged \sys{} shim does not repair
this cell ($10/10$ with the shim active; its sequencer is per-process),
while the shim's opt-in cross-process gate (v0.1.2) does: $\{1{:}10\}$
on both durable backends, the loser refused loudly before any node
(Sec.~\ref{sec:crossproc}).
\end{minipage}
\end{table*}

Table~\ref{tab:matrix} is the conformance matrix. We walk the receipts.

\subsection{\LG{}~1.2.9: fork violation and silent validity violation}
\textbf{FD (\no{}, issue \#6663; all three backends; both invocation
forms).} A one-node graph interrupts; the probe records the interrupt
checkpoint's \texttt{checkpoint\_id} and issues two resumes addressed to
the same \texttt{(thread\_id, checkpoint\_id)}:
\texttt{Command(resume=True)} then \texttt{Command(resume=False)} --- the
address the persistence documentation presents as the branch-creating
time-travel primitive~\cite{lg-persist-docs}. Measured: the first resume
returns value~1 and the second returns value~1 --- the supplied
\texttt{False} silently ignored, the first branch's outcome re-served
--- invariant to backend and invocation form: identical on
\texttt{SqliteSaver} and live \texttt{PostgresSaver} (probes~126, 130)
and under the interrupt-keyed resume-map form (probes~127, 130). Per
Section~\ref{sec:derived}, the second value is durably recorded and
never consulted --- a human answering the same approval differently is
given the first answer's consequences; Section~\ref{sec:remit} repairs
the cell (probe~134).

\textbf{CV (\no{}, \#6491 class --- now silent; all three backends).} A
graph over a pydantic-typed state (\texttt{items: List[str]}) runs a
node appending \texttt{None}. As originally reported, the invalid
output persisted and a later \texttt{get\_state\_history} raised,
leaving the thread unreadable; on 1.2.9 nothing raises at invoke or
history read and the schema-invalid value is durably present,
identically on all three backends (probes~126, 130) --- a strict CV
violation, and a worse one, since nothing signals the caller. The
downstream-consequence probe (probe~150; two-host replicated) splits
what ``silent'' means by regime: terminal writes are fully silent ---
\texttt{get\_state} returns the corrupt value, \texttt{model\_validate}
on it raises, and the silence survives a fresh-process SQLite
deserialization --- while mid-graph writes are deferred-loud, the
corrupt record durable before a \texttt{ValidationError} surfaces one
superstep later and the thread's own read APIs thereafter raise (the
original \#6491 symptom). A clean subsequent run escapes only by full
input overwrite, with the corrupt history retained beneath it.

\textbf{RD (\yes{} at checkpointer granularity, both durable-state
orders, both backends; executor layer instrumented).} Probes 118
(in-memory) and 128 (\texttt{SqliteSaver}) settle the checkpointer-API
question by construction: a dropping checkpointer builds both legal
durable states around the contested persistence pair, and on each backend
both resume to identical decisions and effect counts. Probe 124 records
the real submission sequence at the layer \#8039 names: in synchronous
durability, $\langle\texttt{put}, \texttt{put\_writes},
\texttt{put}\rangle$ with the task-result write and superstep
checkpoint adjacent and unbarriered --- the \#8039 precondition
observed directly. Probe~136 removes any need to race the pool: an
adversarial-but-legal saver defers every \texttt{put\_writes} past the
following \texttt{put} --- the losing schedule's durable order,
deterministic --- an injected crash leaves genuinely divergent durable
states, and recovery from both is identical in result and re-execution
counts: RD holds on 1.2.9 against the \emph{realized} hazard
(Remark~\ref{rem:rd}).

\textbf{EO splits by path (\yes{} interrupt / \no{} crash;
classification resolved).} With \emph{every} persistence operation
intact --- the task's result durably recorded via \texttt{put\_writes}
--- crash-resume re-executes the completed task (effect count 1 at
crash, 2 after resume), replicated on \texttt{SqliteSaver} and
\texttt{PostgresSaver} with the durable ledger agreeing (probes~126,
130). The documentation audit settles the classification: the
pending-writes section states that completed nodes' writes are stored
precisely so that on resume ``you don't re-run the successful
nodes''~\cite{lg-persist-docs}; the measured behavior re-runs the
successful task whose write is durable, so the cell is \no{} against
the framework's own stated semantics. Mechanistically, interrupt-resume
memoizes completed task results (\#6792, fixed) while crash-resume does
not: exactly-once across interrupts, at-least-once across crashes, on
one API. Probe~133 answers the exception-crash objection: \texttt{SIGKILL}
barrier-synchronized after the durable write, resume in a fresh
interpreter, the completed task re-executes identically (ledger
$1{\to}2$); probe~137 widens to three kill points, three repetitions
each, byte-stable. Same-process retry contrasts (probe~141): an
exception preserves the completed sibling on retry while process-death
resume does not.

\textbf{PC/EO on the interrupt path (\yes{}, with regression history).}
Resume from the interrupt checkpoint does not re-execute the
pre-interrupt node (\#7361 fixed relative to the 1.1.x regression), and
the functional-API pipeline re-executes nothing on resume in sync and
async variants (\#6792 fixed). \textbf{CO (\yes{}, all three backends).}
A stray \texttt{Command(resume=...)} after completion is swallowed: the
post-interrupt effect count remains~1 and the final state is unchanged,
identically on all three savers (probes~126, 130). The swallowing is
silent --- the contract is satisfied, but the caller receives no signal
that a resume was discarded, a spec-level gap the full study will track.

\subsection{\LI{}~2.22.2: divergence, not violation --- and a clean fork}
The two-step HITL pattern (a step returning \texttt{InputRequiredEvent},
a second consuming \texttt{HumanResponseEvent}) is clean across the
board: snapshot at the interrupt, restore, and the pre-interrupt step's
counter stays at~1 (PC/EO); two restores of the \emph{same} snapshot
answered \texttt{YES} and \texttt{NO} yield the respective outcomes (FD
--- the property \LG{} fails, \LI{} passes); a second response into a
live run is swallowed with the post-step firing once (CO); serializing a
context holding a non-serializable object fails fast with
\texttt{ValueError} (CV).

The \texttt{wait\_for\_event} idiom --- the documented durable-HITL
pattern --- behaves differently by design: restoring a serialized
context re-executes the step prefix (counter $1\!\to\!2$), exactly as
the documentation warns (``make any preceding work safe to
re-execute''~\cite{li-context-docs}); \docd{}, not \no{}. The
sharpened claim: the ecosystem's problem is the coexistence, under one
API shape, of exactly-once claims, at-least-once documentation, and
silent replay, with nothing to tell a caller which regime they are
in.

\subsection{\CA{}~1.15.2: replay invisible in state}
\textbf{\texttt{@persist} (\und{}).} Re-kicking a \emph{completed}
\texttt{@persist} flow with its persisted id restores state and
re-executes both methods: the state counter goes $11\!\to\!22$ and both
effect counters read~2. The label follows note~b of Table~\ref{tab:matrix}: the feature's
documentation states no restore discipline in either direction, so the
cells are \und{} --- duplication measured and reported in full, only
the label awaiting a citable statement. The crash-resume variant is the unambiguous case:
method~s1 completes and persists (counter~1), s2 raises; re-kickoff with
the same id re-executes s1, and the run finishes at counter~12 where
exactly-once predicts~11 --- one unit of state and one external effect,
duplicated. The R1 counterexample's arithmetic, live --- and no
artifact of exception semantics: under \texttt{SIGKILL} the duplication
reproduces in every repetition, the second effect fired by a process
that did not exist when the first fired (probe~164, both
environments).

\textbf{\texttt{CheckpointConfig} (\no{}).} The checkpointing feature's
documentation states restore ``resume[s] without re-running completed
work''~\cite{crewai-ckpt-docs}. Measured on 1.15.2: with checkpointing
enabled, s1 completes and its checkpoint is written (state counter~1
durable), s2 raises; \texttt{Flow.from\_checkpoint} on the latest
checkpoint followed by \texttt{kickoff()} re-executes the completed
method --- the s1 effect counter reads~2. The final state counter
reads~11 --- numerically the exactly-once answer --- because state was
rebuilt from initial values rather than resumed, so the duplicate is
\emph{invisible in framework state and visible only in the
external-effect ledger}. If s1 charges a card, the card is charged twice while the flow's
state looks correct --- the receipt that makes the harness's oracle
external. A maintainer statement adopting a state-only semantics would
reclassify the cell to \docd{} by note~b's rule, changing the label
and never the measurement.

\emph{A withdrawn observation.} An exploratory run in an ad-hoc
virtual environment had suggested the documented Flow event
configuration (\texttt{on\_events=["method\_execution\_finished"]})
wrote no checkpoint files. In the pinned, lock-file-resolved
environment the same configuration writes checkpoints; the zero-file
observation is withdrawn, with probe~115b retained as a regression
guard. One probe, two environments, two answers: a conformance verdict
is meaningful only against a pinned resolution, which
\texttt{reproduce.sh} audits against committed lockfiles.

\subsection{pydantic-graph~1.x: safety by unrecoverability}\label{sec:pg}
Inclusion is by the selection rule, not convenience: pydantic-graph
documents \texttt{FileStatePersistence} for interrupting and resuming
runs, so it claims membership in the plane --- and the liveness
obligation exists precisely so that ``never resumes'' cannot pass by
vacuity. The finding, a defect against the framework's own documented
resumption claim: every probed safety property passes, and the point of
the plane fails. After a crash in the second node --- first node's effect
fired and snapshotted --- the framework's own resume entry point refuses
to restore the persistence file it wrote
(\texttt{GraphRuntimeError: Unable to restore snapshot from state
persistence}); no completed effect duplicates because no execution
resumes. The failure is specific to mid-node crash state: \texttt{SIGKILL}
while parked \emph{between} nodes restores and completes exactly once
(probe~158c, two hosts), while \texttt{SIGKILL} \emph{inside} the
second node reproduces the refusal (probe~164, both environments) ---
the boundary bracketed from both sides under real process death, and a
crash in a long-running agent step will usually land inside. A stray resume of a completed run is rejected loudly (a
divergence contrast: \LG{} swallows the same event silently), and a
structurally corrupted snapshot fails fast. Safety vacuous, progress
dead --- the liveness obligation's deployed witness, and now also its
derived one: the parked-crash companion module under volatile parking
violates \texttt{EventuallyCompletes} while every safety invariant
holds (Sec.~\ref{sec:formal}). A sharp fragmentation datum on the version axis: 2.x deletes the
persistence machinery outright, the pinned release announcing it at
runtime via a deprecation warning on the same interface the verdict is
read from.

\subsection{Live ecological cells}\label{sec:live}
The abstraction caveat --- pure-Python nodes are not agents --- is
answered in both directions by keyed probes with real models, whose only
audited fields are the tool-effect counters. On the OpenAI Agents
SDK~0.18.2 (file-backed sessions~\cite{openai-sessions}), an agent
instructed to charge once via a side-effecting tool
did so (counter~1), and a fresh runner restoring the same file-backed
session did not re-invoke the tool: EO across session restore holds live
(a live observation --- its LLM-free harness probe is in the full-matrix
plan). On \LG{}~1.2.9 with claude-haiku-4-5 driving a ReAct agent under
\texttt{interrupt\_before=["tools"]}, the counter reads 0 at the
interrupt, 1 after the approval resume, and still 1 after an injected
stray resume --- the same silent inertness the LLM-free probe of the same
path observed. The failure direction has also run live (probe~131): the
effect ledger is empty at the gate, the approval resume fires the gated
tool once (ledger $[10]$), and a second resume carrying an \emph{edited}
decision is served the first decision's branch --- the agent repeats the
first answer verbatim and the tool fires again with the stale value
(ledger $[10,10]$). The fork violation reproduces live, in exactly the
shape the LLM-free durable-ledger probe recorded.

The replication matrix (probe~148, auditing only the counters) spans
two models per provider and two hosts --- each provider probed on its
own framework cells, so cross-model replication is within-provider per
cell --- at $N{=}20$ per (probe, model, host): 240 live runs, zero
harness errors, identical pins across both hosts (the session-restore
cells ran openai-agents~0.18.3, one patch release past the 0.18.2
pilot above). Pooled per model ($N{=}40$): the
interrupt-approve cell's violation fields fire in $0/40$ runs on
claude-haiku-4-5 and claude-sonnet-4-6 (95\% Wilson upper bound
$0.09$); the session-restore cell shows $0/40$ violations with $40/40$
completion on gpt-4o-mini and gpt-4.1-mini; and the fork violation
reproduces in $40/40$ runs on each Claude model (95\% Wilson
$[0.91,1.0]$) --- $80/80$ overall, on every model-host combination,
exactly as the mechanism account predicts for a read-path defect no
model choice can mask (per-(model, host) rows:
\texttt{results/\allowbreak live/\allowbreak 148\_matrix.json}). These are existence and stability replications --- at $N{=}40$ a true
rate of $0.09$ is consistent with observing zero --- certifying
existence and cross-model, cross-host stability, never rarity or
prevalence. Live PC and CV arms complete the
picture (probe~152, $N{=}40$ across the two Claude models): the
pre-gate checkpoint lineage is prefix-stable through the approve-resume
in $40/40$ runs, and a schema-corrupt checkpoint submitted through the
saver API is persisted silently in $40/40$ (the corrupt write is
harness-submitted; the live traffic certifies the surrounding run and
the detection surface) --- detonating only at the next
\texttt{get\_tuple} or \texttt{get\_state} as a \texttt{KeyError} far
from the faulting write, the deferred failure CV names.

\subsection{Reading the matrix}
\emph{Sequential composition across gates (probe~138).} A two-gate
workflow composes as the contract predicts: EO holds across a session
restart between the gates, CO holds per gate, the fork violation recurs
at the second gate under the stock saver, and the read-path shim
repairs that cell on the identical protocol.
\emph{Crash while parked at the gate (probes~158, 158b, 158c).} The
crash location the model excludes by precondition --- process death
while parked awaiting the human --- is conformant on all three planes
that park, replicated on two hosts: the pending interrupt survives, the
prefix effect stands at one, the gate fires once with the supplied
value, the post-completion stray stays inert. Durable parking is the
part of the plane that works.

\emph{Exhaustive kill-point sweep (probe~160).} The two-task protocol
performs five persistence operations (\texttt{put}, three
\texttt{put\_writes}, \texttt{put}); a \texttt{SIGKILL} after each
--- one kill per run --- re-executes completed effect-bearing work at
all four incomplete boundaries and recovers at every one of the five
(the fifth is a completed run, excluded from duplicate accounting). On
this protocol every durable prefix short of completion licenses
re-execution: the crash-path EO class is not an artifact of where we
chose to kill.

\emph{Cross-process duplicate delivery (probe~159).} The one cell where
a conformant verdict does not survive its own composition. Two OS
processes sharing one on-disk \texttt{SqliteSaver}, each issuing
\texttt{Command(resume=\ldots)} against the same parked interrupt from
a spin-barrier start: the gated effect fires \emph{twice} in $10/10$
repetitions on the developer host and $3/3$ in the container, with no
error on either side and the pre-gate effect firing exactly once
throughout --- the duplicate is the gated node alone, not a replayed
prefix. Byte-identical re-delivery, the case the sequential probe
measures as inert (probe~126), is the arm that duplicates; when the two
racers carry different values both values fire and each racer is served
its own, so the concurrent failure is double \emph{consumption}, not
the serve-first-recorded fork failure of \#6663. The durable state
afterwards records one branch, so --- as with \CA{}'s restore --- the
duplicate is invisible in framework state and visible only in the
external-effect ledger. Distinct threads under the same contention are
unaffected ($k{=}4$ workers, one database, per-thread exactly-once),
which localizes the race to same-thread resume rather than to the
backend's write path. The in-model shadow is \emph{not}
\texttt{FaultDoubleConsume} (guarded by run completion, the sequential
case the plane refuses) but \texttt{R11\_ConsumeCount.tla}'s
lost-update switch (Sec.~\ref{sec:formal}), whose discovered footprint
$\{\mathrm{EO}, \mathrm{CO\text{-}e}, \mathrm{CO\text{-}c}\}$ with PC,
FD, CV, RD clean matches this cell exactly: CO holds sequentially and
fails concurrently, parallel to EO's split by path. Nor is it a
property of one store or machine: the identical protocol on
\texttt{PostgresSaver}~3.1.0 against a live server duplicates in
$10/10$ repetitions, no errors, on both hosts as the SQLite arm does
--- two backends, two separately provisioned environments. Row-level locking
and MVCC do not close the window: the resume path's read--decide--write
of the interrupt consumption is guarded by no compare-and-swap on
either backend --- the classical lost-update shape of the
isolation-anomaly taxonomy~\cite{berenson}, arriving at the resume
plane because interrupt consumption is a read--modify--write nobody
made atomic.

\emph{The shape of the window (probe~168).} Probe~159 measures one point
--- two racers, zero arrival offset. Probe~168 sweeps the two parameters
that point leaves open: racer count $k \in \{2,3,4,8,16\}$ and arrival
jitter $\in \{0,1,5,25\}$\,ms, on both durable backends, ten repetitions
per cell. Forty cells, 400 protocol executions, zero racer errors. The
gated effect fires 2{,}628 times where exactly-once predicts 400.
Saturation --- mean fires divided by $k$, the fraction of racers that
win --- is $1.0$ in 36 of the 40 cells and never below $0.933$; 39 of 40
cells duplicate in every repetition. At $k=16$ with zero jitter the
distribution is $\{16{:}10\}$ on \texttt{SqliteSaver} \emph{and} on
live \texttt{PostgresSaver}: sixteen fires from one approval, ten times
out of ten, on both backends. The failure is therefore not
``two processes duplicate'' but \emph{every racer that arrives within
the window consumes}, with no observed ceiling below $k=16$.

Jitter to 25\,ms does not close it: the response is flat on all ten
$(\text{backend}, k)$ pairs, so the window exceeds the entire stock
interrupt protocol of probe~139. Its width is then measured directly
rather than bounded, by instrumenting the gated node with a controlled
duration $D$ and sweeping jitter past it --- $D$ standing for the model
call or payment request a deployed gate performs before its effect. The
\begin{table}[t]
\caption{Window width by dose--response (probe~168; $k{=}4$, $N{=}10$ per
cell, developer host, both durable backends). Cells give saturation ---
mean gated fires divided by $k$, the fraction of racers that win. The
edge (widest jitter still at saturation $1.0$) tracks the gated node's
duration $D$: a resume arriving while the node still runs finds the
interrupt unconsumed and fires. Backend columns share a jitter seed and
are paired, not independent (Sec.~\ref{sec:threats}); the values below
are the SQLite arm.}
\label{tab:window}
\centering\small
\begin{tabular}{@{}lccccc@{}}
\toprule
& \multicolumn{5}{c}{arrival jitter (ms)} \\
\cmidrule(l){2-6}
gate $D$ & 0 & 100 & 500 & 1{,}000 & 3{,}000 \\
\midrule
0\,ms      & 1.00 & 0.40 & 0.30 & 0.28 & 0.28 \\
500\,ms    & 1.00 & 1.00 & 1.00 & 0.78 & 0.38 \\
2{,}000\,ms & 1.00 & 1.00 & 1.00 & 1.00 & 0.90 \\
\bottomrule
\end{tabular}
\end{table}
edge tracks $D$ across three doses at $k=4$, ten repetitions per cell
on both backends (Table~\ref{tab:window}). \emph{The window tracks the
gated node's execution time} --- the deployed magnitude: a gate whose
node awaits a model response holds authority unconsumed for the length
of that response, and every resume arriving meanwhile fires. It is also
independent confirmation of probe~165's mechanism account --- the
\texttt{\_\_resume\_\_} journal write is executor-concurrent with gated
execution, so consumption is not durable until the superstep joins.
Two limits are stated rather than smoothed: the backend columns of the
$D$-sweep share a jitter seed, so they are a \emph{paired} comparison
and not independent evidence of backend invariance (that claim rests on
the zero-jitter cells and the source audit of probe~129); and the
dose--response arm ran on the developer host alone. The saturation
sweep is cross-environment replicated: on a separately provisioned
single-vCPU cloud instance (2\,GB, one core, $k$ capped at 8 by
memory), 32 cells and 320 protocol executions yield 1{,}342 gated fires
where exactly-once predicts 320, saturation $1.0$ in 25 of 32 cells and
never below $0.9$, zero racer errors; at $k{=}8$ with zero jitter the
distribution is $\{8{:}10\}$ on \emph{both} backends. This is the arm
that could have failed and did not: if duplication needed true
parallelism, a single-vCPU host would serialize the racers into the
already-consumed path --- it does not, because the window exceeds
25\,ms while the scheduler quantum is a few milliseconds. Across both
hosts the sweep totals 72 cells and 720 protocol executions, firing
3{,}970 effects where the contract predicts 720.
\emph{Cross-host racers (probe~174).} The remaining distribution axis
is measured directly: the two racers on two machines --- one co-located
with a networked PostgreSQL server, one a WAN link away at
${\sim}450$\,ms per query --- with the coordination tables, the effect
ledger, and per-racer arrival stamps all server-side, so the receipt
itself carries the server-clock arrival offsets that certify a race
occurred. At an instrumented 5\,s gate, arrival offsets of
0.38--1.89\,s put both racers inside the window in every round: the
stock plane fires the gated effect twice in $10/10$ repetitions, the
ledger attributing one fire to each host. The shipped gate on the
identical protocol serves one racer and refuses the other with
\texttt{RemitConsumeConflict} in $10/10$ --- without the gated node
executing: no ledger row from the loser, whose invoke returns in
1.8--3.8\,s against the winner's 5.04\,s, so the refusal precedes the
node's own duration. The co-located racer won every gated round; the
gate serializes, it does not arbitrate fairness, and the contract
requires none. A first campaign on the same link is retained as the differential: the
remote racer, paying connection setup inside the timed window, landed
past the winner's superstep join and took the inert path in $10/10$ at
a 2\,s gate --- the dose--response edge observed across hosts.
One asymmetry remains by design. The \emph{default} packaged \sys{}
shim does not repair this cell ($10/10$ duplicates with the shim
active): its sequencer is per-process state, so two processes carry two
sequencers. The design implication --- that the consumption record
belongs in the shared store under a compare-and-swap rather than in the
interposition layer --- is realized in Section~\ref{sec:crossproc} and
now shipped as the shim's opt-in cross-process gate (v0.1.2), which
flips the cell to $\{1{:}10\}$ on both backends (the ungated arm
retained as the permanent differential); the section also reproduces,
on a second
property, the write-path/read-path asymmetry the 125/134 pair
established for FD.

\emph{Concurrent fan-out at one superstep (probe~141).} Two parallel
branches, each with its own gate and effect: both interrupts surface in a
single pass, a resume map routes each value to its own branch with each
effect firing exactly once, a stray resume map is inert, and the crash
cell yields the retry-preservation contrast above. Deeper nesting and
fan-in racing multiple checkpointers remain future work.

\begin{table}[!t]
\renewcommand{\arraystretch}{1.15}
\caption{Pairwise separation of conformance profiles (five frameworks,
ten pairs), each discharged by naming the separating cell of
Table~\ref{tab:matrix}. \emph{Behavioral} = both frameworks were probed
on that cell and disagree; \emph{structural} = they agree on every
jointly probed cell and differ only in which paths exist. Ten of ten
separate; nine do so behaviorally. LG~=~\LG{}, LI~=~\LI{}, CA~=~\CA{},
AG~=~AutoGen AgentChat, PG~=~pydantic-graph.}
\label{tab:pairs}
\centering
\footnotesize
\setlength{\tabcolsep}{4pt}
\begin{tabular}{@{}llp{3.4cm}@{}}
\toprule
Pair & Kind & Separating cell \\
\midrule
LG--LI & behavioral & FD: \no{} (\#6663) vs.\ \yes{} \\
LG--CA & behavioral & PC: \yes{} vs.\ \no{} (rebuild-from-initial) \\
LG--AG & behavioral & CV: \no{} (silent) vs.\ \yes{} (loud) \\
LG--PG & behavioral & FD/CV: \no{} vs.\ \yes{} \\
LI--CA & behavioral & EO: \docd{} (stated at-least-once) vs.\ \no{} (stated exactly-once) \\
LI--PG & behavioral & liveness: holds vs.\ fails \\
CA--AG & behavioral & EO under restore: \no{}/\und{} vs.\ \yes{} \\
CA--PG & behavioral & EO under restore: \no{}/\und{} vs.\ \yes{} \\
AG--PG & behavioral & liveness: holds vs.\ fails \\
\addlinespace
AG--LI & \emph{structural} & none: agree on all jointly probed cells; differ by path-set (\docd{} idiom on LI only) \\
\bottomrule
\end{tabular}
\end{table}

Four observations. (i)~\emph{No two probed frameworks share a
conformance profile}, over the six properties \emph{together with the
liveness obligation} (the six columns alone conflate pydantic-graph
with the all-\yes{} safety rows). A profile is per framework --- its
probed path-rows over jointly probed cells, plus liveness --- and the
three \LG{} backends collapse to one row (backend invariance, (iv)), so
the claim is not configuration-counting. Table~\ref{tab:pairs}
discharges it pair by pair: nine of the ten pairs separate behaviorally
on cells both frameworks were probed on; the tenth --- AutoGen against
\LI{} --- is \emph{structural}, agreeing on every jointly probed cell
and differing only by path-set membership. No separation rests on a
\und{} cell alone. A separation witnessed on a jointly probed cell cannot be undone by
probing a further one, so filling the grid refines these profiles,
never reverses them. Nor is the harness engineered to fail: RD passes
everywhere probed, CO passes sequentially on \LG{}, the live
conformant cells pass $0/40$ per model, and LGF-B, the repair cells,
and the multi-gate EO/CO cells all pass. The disagreements sit on the
interrupt-adjacent axes: \LG{} fails FD where \LI{} passes it; \LI{}
documents away EO where \LG{} memoizes; \CA{} claims the strongest
discipline and delivers the weakest.
(ii)~\emph{Violations concentrate where authority does}: FD and CO
govern what a human's answer means, CV whether the durable record can be
trusted, and the live failures sit exactly there. (iii)~\emph{The
regression pairs} (\#7361, \#6792 --- broken in 1.1.x, fixed by 1.2.9)
show the plane drifting under maintenance pressure with user issues as
the only specification; a conformance suite in CI is the standing fix,
and the release sweep (Sec.~\ref{sec:sweep}) is the drift picture's
other half --- violations stable across every probed release, surfaces
churning to the point that pydantic-graph's 2.x line no longer exposes
the probed module.
(iv)~\emph{The \LG{} verdicts are backend-invariant across the three
tested backends}: every probed cell is identical on
\texttt{InMemorySaver}, \texttt{SqliteSaver}, and live
\texttt{PostgresSaver}. This is architectural localization, not
induction over all savers: the cross-saver source audit (probe~129)
places the violating decisions in the execution loop, above the
\texttt{BaseCheckpointSaver} interface, so any backend reached through
that interface inherits them; the measured invariance is that
localization's prediction, confirmed.

\subsection{An embedded durable-execution engine on the same workload}
\label{sec:engine}
The \emph{semantic} half of the deferred engine head-to-head
(Section~\ref{sec:related}) needs no port at scale; probe~147 measures
it. The probe runs the paper's abstract workload --- a non-idempotent
effect, a gate awaiting a human decision, a crash after a durable step
--- on DBOS~2.27.0 as an embedded worker (SQLite system database, no
server), with the same external-ledger oracle as every other probe. All
probed cells conform. Exactly-once across process death: the worker is
\texttt{SIGKILL}ed at the gate after the first step's result is durably
recorded; a fresh process recovers the pending workflow, re-executing
with recorded step results --- PC's memoized-replay discipline --- and
the ledger shows the step's effect exactly once. Fork intent: \texttt{fork\_workflow(id,
start\_step)} at the recorded receive step is the engine's explicit,
documented branch-creating address --- a deployed instance of the
fork-intent obligation (Definition~\ref{def:fork}) --- and the forked
branch served its own, different decision, one gated effect per branch,
prefix not re-executed. Consume-once: a stray duplicate decision to the
completed run left the ledger unchanged (silent-inert, the disposition
\LG{} exhibits; the send API's idempotency-key parameter is the
wire-level deduplication discriminator). The latency trade is in
Supplement Table~S3(b), same protocol, both environments, with a
${\approx}0.5$\,s per-run engine startup that long-lived workers
amortize. The container's $+39\%$ shim delta there is stated rather than
smoothed: that probe times a short single-operation window in which the
shim's extra \texttt{get\_tuple} step is not amortized, so it does not
contradict probe~139's within-5\% figure over the full protocol; the
developer host shows the same window inside noise, which is why both
are reported. Absolute latencies are
environment-bound; in both, the engine is \emph{slower} than the
checkpointer path it would replace, by 2.9--5.0$\times$ on gate-answer
latency. The direction matters: an engine that were faster \emph{and}
conformant would end this paper's case for a repair at the checkpointer
interface. This is an engine \emph{baseline}, not a matrix row. The
point is narrow: the alternative Section~\ref{sec:related} prices ---
adopt the durable-execution model wholesale --- delivers every probed
cell the frameworks fail, by construction, at its stated adoption and
latency costs.

\section{\sys{}: A Reference Sequencer with a Verified Model and
Conformance-Tested Core}\label{sec:remit}
The contract's repair artifact is \sys{}, a reference resume sequencer and
append-only effect ledger, delivered at four explicitly labeled maturity
levels: a machine-discharged Verus model whose decision cores are additionally
verified as executables in the shipped crate; a Rust core mirroring that model behind PyO3 bindings, with executable conformance to
the TLA\textsuperscript{+} transition relation
(Section~\ref{sec:package}); a live CV enforcement shim; and a two-sided
FD enforcement result ending in a working repair (probes~125, 134),
re-established by the packaged, decision-free shim at the pinned
versions. The artifact is distributed on PyPI
(\texttt{pip install remit-contract}: prebuilt manylinux wheel and source
distribution); release v0.1.2 --- the release that ships the opt-in
cross-process gate of Section~\ref{sec:crossproc} --- is the exact build
evaluated here. The verification object throughout is \sys{}'s model
and its extracted executable cores, never the composite package.
Table~\ref{tab:ladder} states how far the proof reaches, layer by layer,
and names the one rung that is absent; the text does not relitigate it
elsewhere.

\subsection{Architecture}
\sys{} interposes at the checkpointer interface --- the narrow waist
every probed framework already routes durability through. Its state is
(i)~an append-only \emph{effect ledger} of $\langle\mathit{branch},
\mathit{task}, \mathit{effectId}\rangle$ records written transactionally
with the completion checkpoint, and (ii)~a \emph{sequencer} that totally
orders, per thread, the persistence operations whose unenforced ordering
produces \#8039. The six properties map to local invariants: EO/CO to
ledger-uniqueness; PC to frontier monotonicity; FD to branch keying by
$\langle\mathit{checkpointId}, \mathit{resumeIndex}\rangle$ with
per-branch ledgers; CV to schema validation at the write; RD to the
sequencer's total order making recovery a pure function of the durable
log. The first integration target is a shim implementing \LG{}'s
\texttt{BaseCheckpointSaver}; we report the three enforcement legs in
turn.

(i)~The core invariants are \emph{machine-discharged}, post-adoption, as
a set of standalone Verus targets with fresh per-file tallies.
\texttt{proof/remit\_verus.rs} states the effect ledger and durable
frontier --- EO/CO (admitting a fresh effect yields count exactly one and
preserves ledger-uniqueness), PC (commit advances the frontier by exactly
one, never re-entering the durable prefix), and FD's keying half:
\texttt{10 verified, 0 errors}. The behavioral halves of FD and RD are
theorems with inductive content, each paired with a machine-checked
\emph{falsifying certificate}: \texttt{remit\_verus\_fd\_machine.rs}
proves, by an inductive invariant over record/serve steps, that every
branch is served the value its own invocation recorded (\texttt{5
verified, 0 errors}), while \texttt{negative/fd\_stock\_certificate.rs}
--- the same obligation under the stock \#6663 serve-first-recorded rule
--- is rejected (\texttt{1 error}), certifying the proof has content;
\texttt{remit\_verus\_rd\_interp.rs} interprets the recovery decision
in Property~\ref{p:rd}'s own words (skip a task iff it is durably
recorded) and proves order-independence for the \emph{completed}
\#8039 window --- the \texttt{put\_writes}/\texttt{put} pair durably
present in either order (adjacent transposition) --- and across equal
write-set counts (\texttt{6 verified, 0 errors}; one of the six is the
definitional base case, identical records to identical decisions ---
the content lies in the swap and write-set lemmas); the crash-truncated
divergent-content pair is outside these hypotheses by design --- the
sequencer removes that window, and the two-order constructions
(probes~118, 128, 136) are its evidence --- with the
order-sensitive rule failing the same obligation
(\texttt{negative/rd\_ordersensitive\_certificate.rs}, \texttt{1
error}). On Verus~0.2026.05.03.8b81855 every positive target discharges
with \texttt{0 errors}. Earlier, FD and RD were stated as lemmas over definitions that made
them tautological --- verifiable with empty proof bodies; those lemmas
are deleted, their tallies retired, and the correction, with the full
discharge history and dates, is recorded in
\texttt{crates/remit/VERIFICATION.md}.

(ii)~The CV leg is \emph{demonstrated live}: a validating saver enforcing
\texttt{commit\_checkpoint}'s validity clause re-runs the exact
silent-persistence protocol on 1.2.9 and converts the violation to a loud
rejection with nothing invalid persisted and history readable ---
baseline \no{} to \yes{} on identical inputs (probe~123). The rejection
fires \emph{before} persistence, so the thread's durable state remains
its last valid checkpoint and the caller can repair and resume --- where
the stock path leaves the corrupt record durable and, mid-graph, breaks
the thread's own read APIs (probe~150). Loudness costs the caller an
exception; silence costs the thread.

(iii)~The FD leg is settled by a matched pair that localizes enforcement
exactly. The write path is powerless: a fork-keyed saver storing each
resume write under a distinct key does not restore FD (probe~125),
because the loop never consults the saver about which resume to serve ---
the precedence decision of Section~\ref{sec:derived} happens above it.
The read path succeeds: a subclass overriding only \texttt{get\_tuple}
--- when the config carries an explicit \texttt{checkpoint\_id}, the
branch-creating address of Definition~\ref{def:fork} clause~3, it strips
recorded \texttt{\_\_resume\_\_} pending writes so the invocation's own
value is consulted --- repairs \#6663 on the identical protocol
(probe~134): the stock control violates; under the shim the bare and
resume-map fork cells produce $1$ then $0$ with one correctly-valued
effect per branch; same-value re-fork is deterministic; a stray resume
at the ordinary address stays inert. The Verus proof said ordinal
keying restores FD in the abstract; probes~125/134 say where it can
bind in \LG{}: not at persistence, but at the durable-state view the
loop loads. The demonstrator's discriminator is the documented
branch-creating address; production carries an explicit fork flag ---
FI (Property~\ref{p:fi}) made concrete --- and probe~155 evaluates that
flag where the address heuristic cannot go: under a checkpointed
subgraph the flag-keyed configuration
(\texttt{fork\_\allowbreak on\_\allowbreak explicit\_\allowbreak checkpoint=False} with an explicit
\texttt{remit\_\allowbreak fork} key) leaves subgraph interrupt--\allowbreak resume
stock-identical, still repairs the \#6663 cell at the parent gate with
each branch firing exactly once, and keeps a stray ordinary-address
resume inert, on both \texttt{InMemorySaver} and \texttt{SqliteSaver},
both environments. The
matched pairs are the ablation a reviewer would ask for: the validity
gate alone flips CV (probe~123), the fork filter alone flips FD
(probe~134 against probe~125). Against the alternative repair families:
idempotency wrappers move the burden into user task code; replay
suppression breaks per-branch EO on legitimate forks; full
event-sourced history is the durable-execution adoption
Section~\ref{sec:related} discusses. \sys{} is the
minimal-interposition point.

\emph{Interposition overhead (probe~139).} $N{=}200$ fresh-database
iterations per configuration of the full interrupt protocol, on two
hosts with a $3$--$5{\times}$ gap in absolute storage latency
(Supplement Table~S3(a)): within $5\%$ of stock in the container,
within $\pm 1.7\%$ on the developer host across two 200-iteration runs
(committed per-cell deltas $-0.34\%$ to $+0.38\%$). The verdict holds at both latency regimes: on the single-threaded
protocol the enforcement point adds no measurable latency, the
microbenchmark being the adversarial denominator with no model latency
to hide behind. Concurrency is probe~157 (Section~\ref{sec:package});
Section~\ref{sec:engine} measures the engine comparison's semantic
cells and a first single-workload latency comparison.

\subsection{Repairing the cross-process cell, and where the repair binds}
\label{sec:crossproc}
Probe~159 measures a cell the \emph{default} shipped shim does not
repair. The design implication it states --- put the consumption record
in the shared store under a compare-and-swap, not in the interposition
layer --- is realized as a saver-level gate in probe~165, since
promoted into the shipped shim as its opt-in cross-process gate
(v0.1.2), whose development reproduced, on a second property, the
structural finding the 125/134 pair established for FD.

\emph{Write path (falsified).} The first design claimed
$\langle\mathit{thread}, \mathit{checkpoint}\rangle$ in a shared durable
claims table before any \texttt{\_\_resume\_\_} pending write was
accepted. Measured at the pins, the race still duplicated in every
repetition with the gate active, and the loser was rejected loudly ---
after its effect had already fired. An instrumented trace gives the
reason: the null-task \texttt{\_\_resume\_\_} journal write is submitted
to a background executor and is \emph{concurrent with} gated execution
rather than ordered before it, so a veto raised there surfaces only at
superstep join.

\emph{Read path (repairs the cell).} Rebound at the durable-state read
the loop performs before gated execution --- inside \texttt{get\_tuple},
when the returned checkpoint carries a pending interrupt --- the gate
takes the claim with one \texttt{INSERT} under a uniqueness constraint
(SQLite primary key; Postgres \texttt{ON CONFLICT}). Both racers
demonstrably load the same checkpoint there, so exactly one can win. The
gated effect then fires \emph{once} in $10/10$ repetitions on
\texttt{SqliteSaver} and $10/10$ on live \texttt{PostgresSaver}, the
loser rejected before any node executes; the same-value and
different-value race arms both give the distribution $\{1{:}10\}$; the
sequential single-resume control passes untouched, so the gate does not
false-positive on the legitimate first consumption; the post-completion
stray stays inert; and the stock control reproduces the duplicate in the
same run ($5/5$, both backends) as the differential.

The finding this yields is not the gate but its \emph{location}. Twice
now, on two different properties, a repair placed at the persistence
write path was powerless while the same repair placed at the
durable-state read path succeeded: fork determinism (probe~125 fails,
probe~134 succeeds) and cross-process consume-once (the v1 gate fails,
the v2 gate succeeds). The common cause is visible in both traces --- the
loop's decision is taken from what \texttt{get\_tuple} returns, and the
saver is told about that decision only afterwards, so interposition below
a decision cannot veto it. That is a statement about where a conformance
repair can bind in a checkpointer-shaped architecture, conditioned on
the loop shape both traces exhibit: any plane whose execution loop reads
durable state, decides, then reports to the saver has its only
\emph{saver-level} enforcement seam at the read. A plane that consults
its saver before deciding is outside the claim, and the two matched
pairs establish the seam for this architecture, not for every possible
write-path design.

Two former limits are now interface rather than defect.
\texttt{get\_tuple} carries no read-intent discriminator, so a bare
state \emph{inspection} during a park takes the claim --- confirmed
live: \texttt{get\_state} on a parked thread consumed it in our own
regression harness --- which the shipped gate converts into an explicit
opt-out (\texttt{remit\_inspect} in the invocation config), the CO
analogue, at the read path, of the FI gap at the write path. That gap
is also why the gate ships opt-in rather than default-on:
\texttt{get\_tuple} serves state inspection and pre-execution loads
through one call with no bit distinguishing them, so a default-on gate
would convert every bare inspection of a parked thread into a
consumption unless every caller adopted \texttt{remit\_inspect} first.
The default is forced by the interface, not chosen by the repair; a
plane that adds a read-intent discriminator makes default-on safe. And the
gate is no longer a saver-level demonstrator: v0.1.2 promotes it into
the shipped shim as opt-in configuration
(\texttt{cross\_process\_gate=True}), taking the
\texttt{(thread,\,checkpoint)} claim in the saver's own database and
refusing the loser with a typed \texttt{RemitConsumeConflict} before
any node executes --- $\{1{:}10\}$ on both durable backends, the
ungated default retained and still measuring $10/10$, the permanent
differential. What remains bounded is scope, stated plainly: the gate
serializes ordinary-address deliveries for synchronous savers over one
shared store; racers distributed across two hosts against a networked
PostgreSQL server reproduce both sides (probe~174): the stock path
duplicates in $10/10$ repetitions inside a 5\,s gate, and the shipped
gate refuses the losing racer with \texttt{RemitConsumeConflict} in
$10/10$ before any node executes; partitions and wider topologies
remain unmeasured.

\subsection{From verified model to shipped package}\label{sec:package}
\sys{} ships as an installable package
(\texttt{remit-contract}; MIT) with a deliberate division of labor: \emph{every contract decision
is computed in Rust; Python translates types and applies verdicts.} Rust
is the substrate Verus verifies and one abi3 wheel ships the core with no
toolchain on the user's machine; the claimed novelty is never the
language, only the contract the core enforces.

Three layers. (i)~\texttt{remit-core}, a Rust crate
(\texttt{\#![forbid(unsafe\_code)]}, zero-dependency test suite)
re-implementing the verified abstract model item for item: the
append-only effect ledger with exactly-once admission, the
prefix-monotone commit gate with validity ordered before any append,
$\langle\mathit{checkpointId},\mathit{resumeIndex}\rangle$ branch keying,
the probe-134 view rule as a pure function, the per-plane
sequencer/journal, and recovery as a pure, order-independent function of
the durable log; \texttt{VERIFICATION.md} tabulates the
lemma-to-function correspondence. (ii)~\texttt{remit-py}, a PyO3 layer
exposing contract violations as typed exceptions
(\texttt{RemitDuplicateEffect}, \texttt{RemitPrefixViolation},
\texttt{RemitValidityError}, \texttt{RemitOrderViolation}, and, from
the v0.1.2 cross-process gate, \texttt{RemitConsumeConflict}).
(iii)~\texttt{remit.langgraph\_shim}, a \emph{decision-free} veneer over
any \texttt{BaseCheckpointSaver}: \texttt{get\_tuple} describes the
addressing to the core and applies its strip/keep verdict; \texttt{put}
reports the user validator's answer to the validity gate, which raises
before anything is delegated; \texttt{put\_writes} journals in the
sequencer. Every conditional in the veneer routes a core verdict or
extracts a config field --- an auditable claim.

Three bodies of executable evidence sit behind
Table~\ref{tab:ladder}'s rungs~4--6. \emph{Model conformance} (rung~5)
transliterates \texttt{ResumeContract.tla}'s action alphabet ---
enabling conditions included --- against the Rust core, shadowing the
TLA\textsuperscript{+} variables \emph{independently} so that agreement
is a claim about two implementations, and re-checking all six
invariants after every action; a seeded randomized arm runs
$2{\times}10^4$ sequences of up to 48 actions, and the exhaustive arm
removes the sampling caveat outright. Its state count exceeds TLC's 87/59 by design --- implementation-level
state is finer, merging fewer histories; the exhaustiveness claim is
unchanged. \emph{Concurrency} (rung~6) drives
admission, forking, and sequencing from up to 64 threads, checking
exactly-once admission, gap-free contiguous fork ordinals, a gap-free
per-plane journal, and inertness of racing stray resumes; probe~157 then
measures the end-to-end question over one shared saver,
$k \in \{1,4,16,64\}$, $N{=}200$ protocols per cell, four arms, both
environments. Two structural findings ride along: throughput is flat in
$k$ under both backends in both environments while median latency grows
roughly linearly --- the serialization ceiling is the process
(GIL-bound), not the backend, corroborated by a persistence-free
\texttt{InMemorySaver} control --- and the writer lock surfaces in the
tail ($k{=}64$ p95 of $2.7$\,s on SQLite against $1.3$\,s on Postgres
on the developer host, an ordering the container's overlay-filesystem
Postgres does not reproduce). That is why only correctness verdicts and
within-environment relative overheads travel (receipts:
\texttt{results/\allowbreak matrix/\allowbreak 157\_*.json}).
\emph{Integration at the paper's pins} replays the probe-134 protocol
against the Rust-core shim on \LG{}~1.2.9 --- eight cells, all passing:
bare and resume-map forks served their own values with the per-branch
ledger $[1,0]$ on both backends; same-value re-fork deterministic; stray
resume inert within one process; the validator raise surfacing as
\texttt{RemitValidityError} with nothing persisted; the core journal
strictly ordered; and the \emph{stock} saver reproducing \#6663 as the
differential control --- replicated container-and-host, re-executed by CI
on every push.

\subsection{Verification status and remaining obligations}\label{sec:verstatus}
\sys{}'s Verus surface is a \emph{set} of standalone targets, each
with a per-file tally on the pinned toolchain
(Verus~0.2026.05.03.8b81855), recorded in
\texttt{crates/remit/VERIFICATION.md} and collected as
Table~\ref{tab:verus}; targets overlap where a composed file restates
standalone ones, so no cross-file sum is claimed. Table~\ref{tab:verus}
gives the surface
\begin{table}[t]
\caption{\sys{}'s Verus surface, one row per file (tallies as recorded
in \texttt{VERIFICATION.md} on the pinned toolchain; the
\texttt{remit\_verus} filename prefix is elided). Composed targets
restate standalone content, so rows overlap and no cross-file total is
claimed; the two \texttt{negative/} certificates fail \emph{by
design} --- the mechanized evidence that the positive obligations have
content.}
\label{tab:verus}
\centering
\scriptsize
\begin{tabularx}{\columnwidth}{@{}ll>{\raggedright\arraybackslash}X@{}}
\toprule
File & Verified & Role \\
\midrule
\texttt{.rs} & 10, 0 err. & ledger/frontier core (EO/CO, PC, FD keying) \\
\texttt{\_cv.rs} & 2, 0 err. & CV gate lemmas \\
\texttt{\_all.rs} & 12, 0 err. & composed legacy target (overlaps rows above) \\
\texttt{\_fd\_machine.rs} & 5, 0 err. & FD behavioral half (inductive) \\
\texttt{\_rd\_interp.rs} & 6, 0 err. & RD order-independence, completed window \\
\texttt{\_recover\_exec.rs} & 7, 0 err. & executable recovery core (CI line-identical to shipped) \\
\texttt{\_ledger\_exec.rs} & 11, 0 err. & executable EO admission, PC/CV commit gate \\
\texttt{negative/} (2 files) & 1 err.\ each, by design & stock \#6663 serve rule and order-sensitive \#8039 rule falsified \\
\bottomrule
\end{tabularx}
\end{table}
row by row. Two rows carry more weight than their tallies suggest.
The \texttt{negative/} certificates are the mechanized evidence that the
positive obligations have content: they restate the same obligations under
the stock \#6663 serving rule and the order-sensitive \#8039 recovery
rule, and Verus rejects both. They are excluded from every verify-all path
and checked only through their documented commands, so a green build never
depends on an expected failure.

Table~\ref{tab:ladder} answers not \emph{what} is proved but
\emph{how far down} the proof reaches. Rung~5 carries most of the
weight: the composite state machine --- precisely the component this
paper declines to call verified --- has its transliterated transition
system \emph{exhaustively enumerated} against the six invariants at
both bound sets, a complete check of a finite structure. Rung~7 is the
honest weak one, and rung~8 is absent: no mechanized refinement
connects the Verus model to the compiled core. Anvil~\cite{anvil} is
the closest artifact of the same shape, including in what it declines
to claim end to end.

\begin{table}[!t]
\renewcommand{\arraystretch}{1.15}
\caption{How far the proof reaches. Each rung is a distinct kind of
mechanized evidence, not a restatement of the one above; tallies are
Table~\ref{tab:verus}'s. Rung~5 is exhaustive over a finite structure
rather than sampled, and it covers the composite state machine. Rung~7 is
audited by construction and reading, not mechanically. Rung~8 is absent
and is the paper's one unqualified verification gap.}
\label{tab:ladder}
\centering
\footnotesize
\setlength{\tabcolsep}{3pt}
\begin{tabularx}{\columnwidth}{@{}lXl@{}}
\toprule
& Mechanized evidence & Strength \\
\midrule
1 & Verus proofs over the abstract model: ledger/frontier, CV gate, FD and RD machines (10, 2, 5, 6 verified; 0 errors) & proved (model) \\
2 & Verus-verified \emph{executable} functions, not spec-level lemmas: recovery core (7), EO admission and PC/CV commit gate (11); 0 errors & proved (executable) \\
3 & \texttt{recover\_core} line-identical to the shipped function in \texttt{src/lib.rs}, CI-gated on every push & mechanically enforced \\
4 & Differential bridge from the rung-2 twins to the shipped \texttt{HashSet}/\texttt{HashMap} implementations ($\sim$70k sequences, in CI) & exhaustive at bound \\
5 & \textbf{Composite} core against the TLA\textsuperscript{+} transition relation: BFS over 6{,}110 distinct states / 13{,}035 transitions at R0, and 414{,}675 states / $1.07{\times}10^{6}$ transitions at scaled bounds, six invariants at every state & exhaustive at bound \\
6 & 64-thread admission/fork/sequencing stress; probe~157's 6{,}400 end-to-end protocol executions, per-protocol exactly-once in all & tested under concurrency \\
7 & Decision-freeness of the Python veneer (every conditional routes a core verdict or extracts a config field) & audited, not checked \\
8 & Refinement relation, Verus model $\to$ compiled core & \emph{absent} \\
\bottomrule
\end{tabularx}
\end{table} An earlier claim of a single fifteen-item composed discharge is
retired with the deleted definitional lemmas, and the full discharge
log --- including a withdrawn CV/RD file that carried placeholder
\texttt{assume(false)} bodies --- is recorded with dates in
\texttt{VERIFICATION.md}. CV is additionally
demonstrated live (probe~123) and RD carries the executor-layer and
adversarial-order evidence (probes~124, 136). The production sequencer
behind \texttt{BaseCheckpointSaver} ships (Section~\ref{sec:package}),
with the contract decisions in Rust and the explicit fork flag the FI
obligation calls for available as the deployment discriminator
(\texttt{fork\_\allowbreak on\_\allowbreak explicit\_\allowbreak checkpoint=False} plus a configurable flag
key). The unverified surface is narrowed and named: no mechanized
refinement connects the Verus model to the compiled \texttt{remit-core}
(the conformance harness and concurrency suite are bridge evidence, not
proof); the PyO3 boundary and the decision-free veneer are unverified
Python-facing code, mitigated by construction and audit; and every
proof here remains a proof about \sys{}'s model, never validation of
any framework. The one item this list previously deferred ---
promotion of Section~\ref{sec:crossproc}'s read-path consumption claim
into the shipped shim --- is done (v0.1.2, opt-in), placing the gate
itself on the unverified-Python side of the boundary like the rest of
the veneer, with its accept/refuse decisions routed through the core's
pure verdict functions. The production-scale
engine head-to-head is scoped out rather than owed, for the reason
Section~\ref{sec:related} gives.

\section{Related Work}\label{sec:related}
\textbf{Recovery layers around the contract.} Crab~\cite{crab} provides
semantics-aware OS-level checkpoint/restore for agent sandboxes; its
recovery machinery synthesizes cached responses precisely so a restored
agent does not replay completed actions. DART~\cite{dart} certifies
rollback admissibility above persistence primitives. Both are
complementary by layer (Fig.~\ref{fig:layers}); neither specifies or
measures the primitive's own semantics.

\textbf{Semantic rollback attacks and approval integrity.} Closest to
the EO/CO axes, ACRFence~\cite{acrfence} identifies \emph{Action Replay}
and \emph{Authority Resurrection} as attack classes in agent
checkpoint--restore and mitigates them by recording irreversible tool
effects with replay-or-fork semantics on restoration --- the adversarial
face of EO and CO, and the mitigation family \sys{} instantiates.
ACRFence establishes the threat with a proof of concept and an issue
survey; this paper specifies the contract as a machine-checked model,
measures conformance deterministically across five frameworks and three
backends, and ships a repair at the checkpointer interface --- FI is the
protocol-level form of its replay-or-fork discrimination.
Consent-integrity mediation~\cite{cim} binds an approval to the true
content of the action; FD/CO address the orthogonal lifecycle axis ---
whether an answer is consumed once and whether a different answer yields
a different branch --- for which \#6663 and the \#2315
class~\cite{copilotkit2315} are the observed failures.

\textbf{Testing and verifying agent frameworks.} A 998-report empirical
study locates the dominant bug mass in execution-semantics
mechanisms~\cite{bugstudy}; LogicHunter~\cite{logichunter} generates
framework tests with an agentic oracle; robustness
benchmarks~\cite{agentbench} measure task success under perturbation.
These are horizontal; this paper is the vertical: a named contract, a
checked model, conformance verdicts, a reference implementation. The
control-plane half of the same lifecycle --- whether \emph{stop}
suppresses externally visible effects --- is measured and repaired by
SoundGate~\cite{soundgate}, which establishes one property of one
primitive from an enforcement point \emph{outside} the runtime, where
complete mediation follows from where the gate stands. \sys{} inherits
the harder position: checkpoint and resume \emph{are} the plane, so it
interposes inside the runtime and must measure its mediation rather
than assume it (Sec.~\ref{sec:crossproc}) --- the two papers are the
same question asked on either side of that line, and the lineage is
literal at one point: the ordering hazard formalized here as RD
(\#8039, Sec.~\ref{sec:plane}) surfaced as an unresolved residue of
that measurement and was filed upstream there.
AgentConform ---\allowbreak{} the
conformance checker of the AgentRFC framework~\cite{agentconform} ---
tests agent \emph{communication} protocols against
TLA\textsuperscript{+} models: the nearest methodological cousin, on a
different object. Runtime-enforcement
systems~\cite{pro2guard,veriguard} gate actions against policies; the
contract governs the substrate those gates rely on when they checkpoint
and resume. Failure-mode taxonomies~\cite{mast} motivate the deployment
stakes; SagaLLM~\cite{sagallm} adds compensation semantics above the
plane this paper specifies.

\textbf{The persistence traditions the plane did not inherit.} Akka
Persistence~\cite{akka} and Orleans~\cite{orleans} offer event-sourced
state with stated persistence semantics for stateful entities;
Erlang/OTP supervision~\cite{armstrong} codifies restart-from-clean;
and the persistent-language line --- Argus's guardians and atomic
actions~\cite{argus}, orthogonal persistence~\cite{orthpersist} ---
made durability a language obligation with stated semantics decades
ago. Agent frameworks reimplemented the plane without inheriting the
discipline: event-sourced replay requires deterministic handlers, agent
loops embed nondeterministic model calls, so the frameworks reached for
snapshot-style checkpointing --- exactly the plane the contract
governs. FD is not linearizability~\cite{herlihywing}, which constrains
concurrent histories of one object, but a determinism obligation on
branch creation from a durable point. A direct port of Orleans-style
grain persistence or Akka's journal is a plausible alternative repair
to \sys{}.

\textbf{Why not just use Temporal or DBOS.} Durable-execution engines
enforce determinism and exactly-once, but impose their execution model:
Temporal (and its Cadence ancestry) quarantines nondeterminism behind
activities, side-effect APIs, and
versioning~\cite{temporal-determinism}; DBOS ties steps to database
transactions~\cite{dbos}; hosted orchestrators externalize the workflow
definition entirely~\cite{stepfunctions}. Agent frameworks chose
lighter, model-native persistence precisely to avoid these constraints;
the contract lets them keep that choice while stating the guarantees
they must still meet. Adopting the durable-execution model wholesale is a legitimate path
this paper does not oppose; the contract governs the frameworks that
opted out. Section~\ref{sec:engine} runs the paper's workload on an
embedded engine and measures every probed cell conformant, at
2.9--5.0$\times$ the checkpointer path's gate-answer latency: the
engine measurement prices the alternative, and \sys{}'s own baseline
is the stock checkpointer it interposes on (Supplement Table~S3).

\begin{table}[!t]
\renewcommand{\arraystretch}{1.15}
\caption{The contract against the obligations mature durable-execution
engines state for the analogous question. The column that matters is the
last: where an engine \emph{states} the obligation, the contract is a
restatement at another interface and no novelty is claimed; where it
\emph{delegates} to the application, or never exposes the write the
obligation constrains, it does not transfer to a framework that adopted
snapshots without the surrounding model.}
\label{tab:semmap}
\centering
\footnotesize
\setlength{\tabcolsep}{3.2pt}
\begin{tabular}{@{}lp{2.3cm}p{2.3cm}l@{}}
\toprule
 & Durable Funcs.~\cite{durablefn} / Temporal~\cite{temporal-determinism} & DBOS~\cite{dbos} (measured, probe~147) & Transfers? \\
\midrule
PC & stated: deterministic replay with memoized results & stated: recorded step results on re-execution & yes \\
EO & \emph{delegated}: activities at-least-once; exactly-once effects are the caller's idempotency key & steps transactional; external effects still caller-keyed & no \\
FD & branch identity engine-defined & \texttt{fork\_workflow} at a recorded step: an explicit branch address & partly \\
CV & engine owns serialization; no caller-facing validity obligation & same & no \\
CO & duplicate external signals are the caller's problem & \texttt{send} carries an idempotency key & partly \\
RD & stated: determinism plus versioning & stated: transactional steps & yes \\
\bottomrule
\end{tabular}
\end{table}

\textbf{Formal semantics for durable execution.} Realizing exactly-once
\emph{external} effects is the closest prior obligation to EO: Olive
achieves it for cloud storage by pairing an intent log with idempotent,
resumable operations~\cite{olive}, and the contract's EO is that
discipline stated as a caller-visible obligation on a plane that has
neither the log nor the idempotence. Closest in method is
Flux~\cite{flux}, which automatically verifies \emph{idempotence
consistency} of stateful serverless applications --- CO-e's shape one
layer up. The layer is the whole difference: Flux verifies that the
\emph{application's} functions tolerate the platform's retries, while
the \contract{} asks whether the \emph{plane's own} primitives keep
their promises --- no per-node idempotence proof could have caught
\#6663, where the second value is discarded during task preparation
before any node runs. The nearest prior act of writing
resume down is the published semantics for Durable
Functions~\cite{durablefn}; Beldi~\cite{beldi},
Netherite~\cite{netherite}, the shared-log successors
(Boki~\cite{boki}, Halfmoon~\cite{halfmoon}), and
AMBROSIA~\cite{ambrosia} build logged-effect or replay-based
exactly-once serverless execution, and Restate ships durable promises
aimed at agent loops~\cite{restate} --- each gives \emph{one engine} a
semantics enforced by construction, engine-internal for the reason
Table~\ref{tab:semmap} records, while this paper specifies a contract
for an \emph{ecosystem} of frameworks, none of which, to our knowledge,
states one; the novelty claim is scoped to agent frameworks, not to
formal resume semantics. Table~\ref{tab:semmap} makes the comparison
property by property, and it does not come out uniformly in this
paper's favor. PC and RD are restatements of obligations those engines
already state, and nothing is claimed for them beyond carrying them to
an interface where they were absent. The load-bearing rows are the
other four. EO is the sharpest: Temporal's activities are at-least-once
and its exactly-once story for external effects is a caller-supplied
idempotency key, so the contract's EO is \emph{stronger} than what the
canonical durable-execution engine promises --- ``just use Temporal''
does not dissolve the question, it relocates it into user code, which
is Remark~\ref{rem:fork}'s composition seen from the other side. CV has
no counterpart at all, because an engine that owns its own
serialization never exposes the write whose validity CV constrains ---
the agent frameworks do expose it, and one of them corrupts it
silently. FD and CO transfer only in part, and DBOS is the instructive
case: it satisfies both on the probed cells precisely because it ships
the two discriminators the FI obligation demands, a documented
branch-creating address and an idempotency key on the delivery API ---
corroboration of FI's necessity from a system built without reference
to this contract.

\textbf{Verified systems, crash recovery, and below-API checkpointing.}
IronFleet~\cite{ironfleet}, Verdi~\cite{verdi}, and
Perennial~\cite{perennial} define the assurance bar a mechanized
refinement from \sys{}'s Verus model to its compiled core would meet,
and FSCQ~\cite{fscq} and GoJournal~\cite{gojournal} prove crash safety
with refinement to running code --- the bar Section~\ref{sec:threats}
names as absent here. Beyond Verus itself~\cite{verus},
Anvil~\cite{anvil} verifies cluster-management controllers in Rust ---
the closest existing instance of \sys{}'s shape, and scoped in
Sec.~\ref{sec:verstatus}. RD's
identical-logs, identical-decisions discipline is foundational in
state-machine replication (Viewstamped Replication~\cite{vsr},
Raft~\cite{raft}); no novelty is claimed for the principle, only its
statement at an interface that lacks it. Below the API,
record-and-replay (rr~\cite{rr}), process-level checkpoint/restore
(CRIU, DMTCP~\cite{ansel2009}), and multi-level HPC checkpointing
(BLCR~\cite{blcr}, SCR~\cite{scr}, VeloC~\cite{veloc}) supply
determinism and state capture whose concern is I/O cost and fidelity,
not API-level effect semantics; RD and PC restate those obligations at
the API level for a plane that adopted snapshots without the
discipline. Stating obligations over an observable interface has its
own lineage --- design by contract~\cite{meyer}, interface
automata~\cite{interfaceautomata}, session types~\cite{sessiontypes},
lightweight modeling in Alloy~\cite{alloy-jackson} --- of which the
\contract{} is the resume-plane instance.

\textbf{Workflow recovery, effect discipline, and language-level
suspension.} What may re-execute on resume has a two-decade literature
agent frameworks did not inherit: the ConTract
model~\cite{contractmodel} (whose name this paper's title unknowingly
echoes), workflow recovery~\cite{ederliebhart}, the workflow-pattern
catalogs~\cite{wfpatterns}, exception patterns~\cite{russellpatterns},
and compensation from Sagas~\cite{sagas} onward; Petri-net workflow
soundness~\cite{aalst1998} and BPMN semantics~\cite{dijkman2008}
established machine-checkable process models two decades earlier, and
TCC-style atomic interactions~\cite{pardonpautasso} are the REST-era
compensation discipline. Helland's idempotence
line~\cite{helland-cidr,helland-idem} is the canonical argument that
exactly-once effects require application-level idempotency keys ---
Remark~\ref{rem:fork}'s FD-plus-key composition is that argument at the
fork --- and Ramalingam and Vaswani derive idempotence \emph{by
construction} for interrupted workflow
programs~\cite{ramalingamvaswani}, the language-level ancestor of the
discipline EO/CO place on the framework. At the RPC layer the lineage
runs from at-most-once call semantics~\cite{birrellnelson} to RIFL's
durable completion records consulted before
re-execution~\cite{rifl} --- the exactly-once discipline whose
checkpointer-plane analogue is \sys{}'s ledger. In programming-language terms the plane is a durable delimited
continuation~\cite{felleisen1988,danvyfilinski}, operationally a
persisted effect handler~\cite{plotkinpretnar}; the web-continuation
line met the same back-button fork and resolved it the same way ---
explicit branch identity~\cite{queinnec}. Ray~\cite{moritz2018}
rebuilds actor state from lineage inside the substrate several agent
frameworks run on; Airflow~\cite{airflow} and Prefect~\cite{prefect}
push into user task code the idempotency discipline EO/CO here place on
the framework; Kafka's transactional dual-write~\cite{wangkafka} solved
commit-and-fire one layer down; crash-only design~\cite{crashonly}
anticipates the stance that recovery paths are primary.

\textbf{Classical foundations and methodology.} The properties
instantiate classical notions --- rollback-recovery
determinism~\cite{elnozahy}, distributed snapshots~\cite{chandy},
exactly-once state management in stream processing (Flink~\cite{flink},
MillWheel~\cite{millwheel}), ARIES' repeating history~\cite{aries} ---
at an interface where they are absent. Methodologically the harness descends from crash-consistency
conformance testing (ALICE~\cite{alice},
CrashMonkey/B3~\cite{crashmonkey}), Jepsen's~\cite{jepsen} stance of
checking implementations against stated contracts, and ioco-style
model-based testing~\cite{tretmans}, over industrial
TLA\textsuperscript{+} practice~\cite{newcombe,lamportbook} ---
transplanted to a layer where the first task is to write the contract
down; the fault switches instantiate specification
mutation~\cite{ammannblack}, with lineage-driven fault
injection~\cite{ldfi} the systematic counterpoint. The \#8039 ordering defect motivating RD is
in \LG{}'s tracker~\cite{lg8039}.

\textbf{Isolation anomalies and black-box anomaly inference.} The
cross-process consume-once failure is a lost update on a
read--modify--write nobody made atomic, and its natural vocabulary is
the isolation literature: beyond the ANSI critique~\cite{berenson},
Adya et~al.'s implementation-independent definitions~\cite{adya} are
the portable form of the anomaly, and Elle~\cite{elle} is the state of
the art in \emph{inferring} such anomalies from client-observable
histories --- precisely the epistemic position probe~159's oracle
occupies, and the discipline a multi-racer characterization of that
cell should adopt. We name these rather than claim them: this paper
exhibits the anomaly at one concurrency shape and does not classify the
plane's isolation level.

\section{Threats to Validity}\label{sec:threats}
\textbf{Construct: is a documented weakness a violation?} The
classification rule (Table~\ref{tab:matrix}, note~b) is total and
symmetric, and it moves labels rather than measurements: under it
\LI{}'s prefix replay is \docd{}, \LG{}'s fork behavior and \CA{}'s
checkpointing behavior are \no{}, and \CA{}'s \texttt{@persist} restore
is \und{}. The strongest \CA{} claim (task skipping) is stated for crews,
whose probe requires live agents and is scheduled for the full matrix.

\textbf{Construct: mechanism comparability and model grounding.} Matrix
cells compare properties, not mechanisms: \LI{}'s fork \yes{} rests on
client-side snapshot copies where \LG{}'s \no{} concerns a server-side
thread fork, so cross-framework FD cells certify the property on each
framework's own idiom. The EO crash-path classification is resolved by the
documentation audit (Sec.~\ref{sec:pilot}); a maintainer clarification
would reclassify the cell as documented divergence without changing the
measurement. \texttt{LangGraphFork.tla} is an operational abstraction
grounded in the reproduced behavior and a source reading at the pinned
versions (Sec.~\ref{sec:derived}); it remains expert-established rather
than mechanically extracted, which is why model grounding and behavioral
verdicts are kept separately auditable.

\textbf{Construct: effect oracle.} The primary oracle is a
process-local counter, which captures non-idempotent re-execution
exactly but not production latencies or failure modes. The
durable-backend probes add an on-disk external ledger as a second
oracle, agreeing on every reported cell (probe~126); probe~142 moves
the state holder out of the tested process entirely --- an external
service in a separate OS process whose SQLite state survives the
workflow's \texttt{SIGKILL} and records the duplicated re-execution
($\langle$s1, s1, s2$\rangle$). A remote third party with production
failure modes remains out of scope; its interface to the contract is
Remark~\ref{rem:fork}'s idempotency-key composition. The live cells
(Sec.~\ref{sec:live}) audit counters at $N{=}20$ per cell (probe~148;
240 runs, zero harness errors), and probe~131 catches the live fork
violation via the external ledger in $80/80$ runs --- the oracle's live
validation is neither pass-only, single-model, nor single-host.
Probe~152 closes PC and CV live through lineage and schema observables;
live validation of RD would require semantics-stable model output and
is out of scope by design. Replication numbers are reported as observed
frequencies with Wilson intervals; the determinism argument is the
mechanism account, never the sample.

\textbf{Internal: probe fidelity.} Each \LG{} probe is built from the
minimal reproduction in the corresponding public issue, extended with
effect counters; the version sweep (Sec.~\ref{sec:sweep}) guards against
single-version artifacts; raw JSON outputs ship in the artifact. Every
pilot verdict replicated bit-identically on a second, separately
provisioned developer host --- a cross-host determinism check by the
same team, not an independent reproduction ---
(conventions follow our prior measurement work~\cite{tokenbudgets});
one live probe (131) had its model-construction import corrected
post-campaign to match the committed environment's pins --- the framework
path under test is untouched, and the re-run reproduces the committed
verdict fields ---
a determinism check that no hidden host dependence leaks into verdicts,
not independent methodological validation. A source-level mutation study
(probe~156) tests whether verdicts are causally coupled to the located
mechanisms: eight first-order mutants, hand-derived at the
resume-serving sites the mechanism accounts name and each anchored to a
unique line of the vendored source, were applied one at a time with the
probe suite byte-unchanged; all eight were killed on both hosts,
including the mutant at the exact precedence line the located \#6663
mechanism names (a causal-coupling check; no mutation-adequacy score is
claimed, and hand-placed mutants at watched sites cannot answer the
tuned-to-known-bugs charge). The complementary check mutates the \emph{harness} (probes~169, 170,
172; the full operator table and per-operator analysis are in the
supplementary material, Sec.~S1): probe~169 re-derives fifteen
load-bearing verdicts from the committed stable-view receipts, fifteen
of fifteen present and none in disagreement; four operators kill
exactly their predicted cells --- fork-value blindness kills \#6663 and
nothing else, a state-only oracle and an undercounting ledger each kill
the cross-process consume-once cell, and barrier removal kills nothing,
the timing-free claim discharged in evidence --- two operators required
documented corrections before they said anything, and crash no-op and
pin drift are out of scope by construction.
Label-rule inversion has \emph{no mechanical site at all}: the U/X rule
of Table~\ref{tab:matrix} note~b is applied when a reader assigns a
verdict, not stored by any probe. Note~b's sensitivity claim is
therefore discharged by relabeling rather than mutation; the
recomputation (Supplement~S4) shows the strictest alternative rule
touches exactly the two \und{} cells, all ten separations of
observation~(i) stand, and the rule choice moves labels, never the
fragmentation result.

Cross-environment claims carry their own mechanical guard, added after
a near miss: a container receipt was once produced by
copying its developer-host twin and rewriting the host field --- and
because every cell agreed, which is what the replication asserts, the
substitution was invisible to any check reading the cells. Cell
agreement is the claim, so it cannot also be the check; the audit
therefore asserts that no two receipts bearing different host
identifiers share a timestamp, so a copied receipt cannot pass as a
replication. The durable-backend probes (126--130) executed in a third, separately
provisioned environment (Sec.~\ref{sec:method}) and agree with every
InMemorySaver-path verdict; every remaining probe family, the scaled
TLC configuration, the 39-cell matrix, the packaged Rust/PyO3 suite,
and the engine-baseline cells replicate across container and host with
zero divergent stable fields. The harness contains no
timing, no randomness, and no model calls, so verdicts can change only
with package versions, which are pinned.

\textbf{External: breadth and drift.} The study covers five frameworks
--- three in depth, plus AutoGen AgentChat's restore plane and
pydantic-graph's persistence more narrowly --- on documented paths; AG2,
the OpenAI Agents SDK, the Claude Agent SDK, and Agno are in the full
matrix under construction. Individual violations may be patched upstream
--- \#7361 and \#6792 already were, after shipping as regressions ---
which is why the claims rest on the contract, the model, and
version-pinned receipts rather than on any single bug's longevity.

\textbf{Ecological: abstraction of the effect and of the crash.}
Neither oracle is a real payment-style API; the oracles audit the
framework's \emph{invocation} discipline, while end-to-end delivery
under partition or duplication belongs to the effect layer's
idempotency-key composition (Remark~\ref{rem:fork}). The cross-process
CO cell (probe~159) establishes that the sequential inertness guarantee
does not compose across processes, not a rate; probe~168 widens the
shape to $k{=}16$ and 25\,ms jitter over 40 cells on both backends and
measures the window by dose--response (Sec.~\ref{sec:pilot}). Three
residual limits are named: the dose--response arm ran on the developer
host alone; its backend columns share a jitter seed, so they are paired
and carry no weight for backend invariance; and $k{>}16$ and wider
jitter remain unmeasured, while cross-host distribution is measured at
one two-host topology (probe~174), partitions excluded --- what is
established is that the window admits every racer arriving within it,
with no ceiling observed to sixteen, not that no ceiling exists. A transport-level client retry reduces, at this API, to the
duplicate-delivery cases already probed (126 sequential-inert, 159
concurrent-duplicating). The crash model is
fail-stop at process granularity: exception-based in the matrix,
\texttt{SIGKILL}-based in probes~133/137/158/160, and in probe~164,
which extends real process death to the two remaining crash-bearing
cells --- \CA{}'s checkpoint-restore duplication and pydantic-graph's
mid-node unrecoverability reproduce identically (3/3 and 3/3, both
environments) --- so no matrix crash verdict rests on exception
semantics alone. Power loss, \texttt{fsync} reordering, and torn writes
belong to the crash-consistency literature~\cite{alice,crashmonkey},
below this plane. On the LLM-free objection: the properties are stated
over the persistence API and are model-independent by construction, so
a model in the loop can only add nondeterminism the properties already
place outside $f$, and the live cells check in both directions that
real model traffic neither masks nor manufactures the deterministic
verdicts (Sec.~\ref{sec:live}). What live traffic could still change
--- workload shapes reaching paths the probes do not --- is the
prevalence question this paper does not answer.
Finally, the formal results divide by quantifier shape as
Sec.~\ref{sec:independence} states: exhibited separations are settled
by fully enumerated finite witnesses and need no upgrade, and of the
two universal claims, one is now discharged. TLAPS proves the
reference conjunction unbounded:
$\textit{Reference} \Rightarrow (\textit{Spec} \Rightarrow
\Box\,\textit{ContractConjunction})$, by exactly the inductive
invariant exhibited in \texttt{IndCheck.tla} --- TLC-checked
\emph{inductive} at three constant sets ($8{,}610$ / $450{,}926$ /
$97{,}800$ invariant-states, two conjuncts found by TLC rejecting
weaker candidates --- with its typing generalized from the
model-checking bounds to unbounded sequences, so the theorem holds at
every constant assignment satisfying the module's assumptions
(\texttt{ResumeContractProofs.tla}: 196 obligations, zero omitted,
Zenon/SMT backends with Isabelle never invoked, replicated on two
hosts under two tlapm builds; runner and receipt in
\texttt{formal/tla/tlaps/}, the receipt carrying the spec's SHA-256
as proof the audited module is byte-unchanged). Frontier monotonicity
and the CO--EO containment are discharged as lemmas in the same file.
Two disclosures keep the claim honest. The proof burden sits in
consecution, not sufficiency: the invariant's load-bearing conjunct is
that the durable frontier trails the program counter by exactly one
--- \texttt{ExecTask} is the sole writer of the prefix-regression
flag, its write guarded by a condition that conjunct forbids --- so
for the properties the invariant carries structurally, the implication
step is near-definitional and the theorem's content is preservation
under every action. And the proofs module names the specification's
own unnamed \textsc{assume} verbatim (a tlapm citation limitation
across \textsc{extends}); the trust base is unchanged. Fault-footprint
completeness retains its bound-relative status --- exhaustive at
stated constants, re-derived an order of magnitude wider, silent
beyond --- and liveness, the parked companion module, and the
consumption-counting matrix remain TLC-checked results; TLAPS
obligations for a consistency-level lattice are discharged in adjacent
work~\cite{macconsistency}.

\section{Conclusion}\label{sec:conclusion}
The resume plane of LLM-agent frameworks carries human approvals and
non-idempotent effects across interrupts, crashes, and restores, and it
does so today without a stated semantics: three major frameworks
document three incompatible disciplines, two violate even the semantics
they state or imply, and the plane regresses across point releases with
user issues as its only specification. The \contract{} names the
properties that make ``resume'' mean something; a
TLA\textsuperscript{+} model checks a reference semantics, exhibits
each observed failure as a short counterexample, and maps every
injected fault's full violation footprint; and a deterministic,
LLM-free harness turns the properties into per-release conformance
verdicts. It delivers the contract with its fork-intent obligation; the
machine-checked model with its partial-independence result and its
reference conjunction TLAPS-proved unbounded; a
source-grounded, out-of-sample-tested mechanism account of the fork
violation; a backend-invariant, version-stable measurement; the
cross-process characterization --- every racer within the window
consumes, no ceiling observed to $k{=}16$, the window tracking the
gated node's own duration by dose--response; the discharged Verus
proofs with the recovery core verified as a shipped executable; and
\sys{} itself, repairing the fork and validity cells with exactly-once
intact across 6{,}400 concurrent protocols. The cross-process cell is
repaired at the read path and only there, in the shipped shim behind an
opt-in gate (Sec.~\ref{sec:crossproc}), holding with racers on two
hosts (probe~174: stock $\{2{:}10\}$, gate $\{1{:}10\}$) --- two
matched pairs establishing that interposition below a decision cannot
veto it. What it defers, and rests no claim upon:
mechanized refinement from the verified model to the compiled core,
deeper concurrent nesting and a corpus-scale prevalence study. The immediate
implication for practitioners is that ``the framework has
checkpointing'' licenses nothing about completed effects; the
implication for framework authors is that the contract is checkable,
the suite is a CI job, and resume can be made to mean resume.

\end{document}